\documentclass[preprint,11pt,letterpaper]{elsarticle}

\usepackage[letterpaper,margin=1in]{geometry}
\usepackage{setspace}
\usepackage{amsmath,amssymb,amsfonts,bm}
\usepackage{amsthm}
\usepackage{algorithm}
\usepackage{algorithmic}
\usepackage{graphicx}
\usepackage{enumitem}
\usepackage{textcomp}
\usepackage{xcolor}
\usepackage{tikz}
\usetikzlibrary{calc}
\usepackage{subcaption}
\usepackage{multirow}
\usepackage{booktabs}
\usepackage{placeins}
\usepackage{float}
\usepackage{caption}
\usepackage[colorlinks=true,linkcolor=blue,citecolor=blue,urlcolor=blue]{hyperref}
\hypersetup{breaklinks=true}

\journal{Neural Networks}

\newtheoremstyle{bfplain}%
  {3pt}{3pt}%
  {\itshape}{}%
  {\bfseries}{.}{.5em}{}%
\theoremstyle{bfplain}
\newtheorem{lemma}{Lemma}
\newtheorem{proposition}{Proposition}

\begin{document}

\begin{frontmatter}

\title{Streaming Adversarial Robustness in Fuzzy ARTMAP: Mechanism-Aligned Evaluation, Progressive Training, and Interpretable Diagnostics}

\author[inst1]{Shane Cairns}
\ead{sacg3q@mst.edu}

\author[inst3]{Leonardo Enzo Brito da Silva}
\ead{leonardo.brito@imd.ufrn.br}

\author[inst1,inst2]{Sasha Petrenko}
\ead{sasha.petrenko@mst.edu}

\author[inst1,inst2]{Donald C. Wunsch II}
\ead{wunsch@mst.edu}

\author[inst1,inst2]{Jian Liu\corref{cor1}}
\ead{jliu@mst.edu}

\cortext[cor1]{Corresponding author.}

\affiliation[inst1]{organization={Department of Electrical and Computer Engineering},
            addressline={Missouri University of Science and Technology},
            city={Rolla},
            state={MO},
            country={USA}}
            
\affiliation[inst2]{organization={Kummer Institute Center for Artificial Intelligence and Autonomous Systems (KICAIAS)},
            addressline={Missouri University of Science and Technology},
            city={Rolla},
            state={MO},
            country={USA}}

\affiliation[inst3]{organization={Instituto Metr\'opole Digital, Universidade Federal do Rio Grande do Norte},
            addressline={Natal, RN 59078-900},
            country={Brazil}}

\begin{abstract}
Adversarial robustness has been studied extensively for offline deep networks, but much less is known about how attacks, defenses, and reliability signals behave in neural learners that update through strict single-pass streaming. This paper studies this problem in Fuzzy ARTMAP, an Adaptive Resonance Theory architecture whose decisions are governed by winner-take-all category competition, complement coding, match tracking, and replay-free prototype updates. We introduce WB-Softmax, a differentiable white-box attack surrogate aligned with ARTMAP's category-competition and map-field prediction mechanism, and we formalize a streaming evaluation principle requiring robustness to be assessed on the final deployed model rather than on stale intermediate states. We further examine replay-free adversarial training under streaming-compatible protocol choices, including offline versus online attack generation, selective updating, and progressive training. Across four image benchmarks, WB-Softmax provides a strong adaptive white-box evaluator, achieving 89--100\% attack success on vanilla Fuzzy ARTMAP models across the evaluated benchmarks. We show that defense rankings can reverse across evaluation protocols: offline adversarial training may appear strong under transfer attacks yet collapse under adaptive white-box evaluation, whereas progressive two-stage selective training achieves the strongest overall replay-free robustness. Finally, we show that ART's explicit category geometry supports interpretable diagnosis of structural and reliability failures, including \emph{separation collapse}---a failure mode in which different-class categories become increasingly overlapping during adversarial adaptation---and a reversal in match-score ordering after selective adversarial training. These results establish a mechanism-aligned, protocol-aware framework for adversarial robustness in streaming prototype-based learners.
\end{abstract}


\end{frontmatter}

\section{Introduction}
\label{sec:intro}

Adversarial robustness is now a central requirement for neural learning systems, yet most existing methodology assumes a fixed model trained offline with repeated access to historical data. This assumption leaves a major gap for streaming neural architectures that must learn in a single pass, update their decision structure online, and operate without replay. In such systems, adversarial robustness is not only a question of perturbation size or attack strength; it also depends on when adversarial examples are generated, which evolving model state they target, and whether internal reliability signals remain meaningful after adaptation.

Since adversarial examples were first documented~\cite{szegedy2013intriguing}, a large literature has developed stronger attacks, robust training methods, and evaluation protocols~\cite{goodfellow2014explaining,madry2017towards,carlini2017towards,tramer2018ensemble,croce2020reliable,andriushchenko2020square,zhao2025holder}. This literature has also emphasized pitfalls such as gradient masking and the need for strong adaptive evaluation~\cite{athalye2018obfuscated,uesato2018adversarialrisk}, while RobustBench and recent surveys have helped standardize empirical robustness reporting and organize the broader adversarial-learning landscape~\cite{croce2021robustbench,qian2022survey,zhao2024advtraining_survey}. However, it remains centered on offline training with repeated access to historical data. In particular, both adversarial training and robustness evaluation usually assume either repeated optimization over the same dataset or repeated access to prior samples, assumptions that do not hold in strict single-pass streaming.

Recent work has begun examining the intersection of continual learning and adversarial robustness~\cite{kirkpatrick2017overcoming,parisi2019continual,delange2022survey}. At the same time, much of the broader continual-learning literature relies on replay, episodic memory, rehearsal, distillation, or other forms of repeated optimization~\cite{li2018lwf,lopez2017gem,chaudhry2019agem,rebuffi2017icarl}. Existing robust continual-learning methods likewise typically rely on replay, regularization, or memory-augmented training~\cite{mi2025adversarial_continual,bang2022online}. These mechanisms are valuable in their intended settings, but they are difficult to reconcile with strict single-pass streaming, where each sample is processed once and then discarded. Thus, much of the current continual-robustness literature addresses robust continual learning with memory, rather than robustness under the no-replay constraints that motivate truly streaming models.

Adaptive Resonance Theory (ART) networks~\cite{grossberg2013adaptive,grossberg2021conscious,dasilva2019survey} provide a particularly important setting for studying this gap. ART networks were designed for stable incremental learning, and Fuzzy ARTMAP remains one of the most established supervised ART architectures for single-pass classification. Unlike conventional deep networks, Fuzzy ARTMAP predicts through explicit category competition and map-field assignment, while learning proceeds through match tracking, category creation, and fast prototype updates. These mechanisms make it possible to diagnose robustness failures directly from internal category geometry, but they also make standard deep-network adversarial evaluation insufficient.

Fuzzy ARTMAP~\cite{carpenter1992fuzzy}, the supervised variant of ART, is especially relevant because it supports incremental classification through complement coding, match tracking, and explicit map-field supervision. Recent work has expanded the ART ecosystem through modular software implementations~\cite{melton2025artlib}, deep hierarchical extensions~\cite{melton2025deepartmap}, gradient-free deep-learning formulations inspired by ART dynamics~\cite{petrenko2025deepart}, and analysis of computational trade-offs induced by match-tracking mechanisms~\cite{melton2025matchtracking}. Recent robustness work on prototype-based or nonstandard models has also considered hyperspherical prototypes, discriminative prototype learners, and metric-learning perspectives~\cite{mygdalis2022hyperspherical,huai2022metric,sabzevar2025advdpnp,snell2017prototypical}. However, these studies do not address ARTMAP's particular combination of complement coding, winner-take-all category competition, fast one-pass category updates, and explicit match-based internal scores. This leaves open a distinct question: how should adversarial robustness be defined, attacked, trained, and interpreted in Fuzzy ARTMAP under its native strict-streaming regime?

A preliminary version of this work appeared in IJCNN 2026~\cite{Cairns2026IJCNN}. This journal version extends it from an empirical robustness study to a broader mechanism-aligned framework, adding a formal final-model streaming evaluation principle, reliability analysis of match-score inversion, geometry-based diagnosis of separation collapse, separation-aware training analysis, and derived unconditional robustness results.

This gap is not only empirical. Core ARTMAP operations---winner-take-all category competition, complement coding, and piecewise fast-learning updates---make standard gradient-based white-box attacks poorly matched to the model's actual prediction mechanism. Consequently, robustness conclusions can be misleading unless the attack objective is explicitly aligned with ARTMAP's competition-and-mapping structure. A second challenge is specific to streaming learning itself. In strict single-pass training, category boundaries evolve continuously as categories are created, absorbed, or reset. Adversarial examples generated against an earlier model snapshot can therefore become stale relative to the final deployed classifier. In this work, we treat this threat-model issue as a methodological question rather than only an implementation detail: in streaming learners, robustness should be assessed on the final streamed model using adaptive attacks crafted against that final state, rather than inferred from stale, transfer-only, or partially aligned perturbations.

These questions also arise beyond ARTMAP. Fuzzy ARTMAP is a particularly informative testbed because it combines strict single-pass adaptation, explicit prototype competition, interpretable internal geometry, and replay-free learning. This makes it well suited for studying a broader methodological problem: how adversarial robustness should be evaluated and improved in streaming prototype-based learners whose decision boundaries evolve continuously during deployment.

This paper studies adversarial robustness in streaming Fuzzy ARTMAP from three coupled perspectives: evaluation, training, and interpretability. More specifically, we ask three questions. First, how should adaptive white-box attacks be defined for a non-differentiable winner-take-all learner whose predictions are formed through category competition and map-field assignment? Second, under strict single-pass no-replay constraints, what training and evaluation protocol is required to avoid robustness overestimation caused by stale attacks on intermediate model states? Third, can ART's explicit category geometry be used not only to interpret post hoc behavior, but also to diagnose robustness failure modes online and motivate targeted replay-free interventions? This final question is also connected to deployment reliability: if internal scores are reused for rejection, abstention, or escalation, then their semantics must remain valid after adversarial training rather than only in the vanilla regime~\cite{chow1970reject,geifman2017selective,geifman2019selectivenet,hendrycks2017baseline}.

To answer these questions, we combine three methodological components. First, we develop WB-Softmax, a differentiable softmax relaxation that aggregates category-level choice values into class-level scores aligned with the ARTMAP map field, enabling strong adaptive white-box evaluation. Second, we distinguish offline and online adversarial-example generation under streaming updates, and compare standard, selective, and progressive two-stage training rules within the same replay-free setting. Third, we exploit ART's explicit category geometry through iCVI monitoring and overlap-based diagnostics to identify structural failure modes; we propose a separation-aware update rule as the first concrete intervention motivated by these diagnostics, and characterize its operational behavior including a structural limitation of overlap-only gating. Taken together, these components define a mechanism-aligned, protocol-aware, and interpretability-driven framework for adversarial robustness in strict streaming learners.

Our contributions are threefold. First, we establish a mechanism-aligned evaluation framework for adversarial robustness in streaming prototype-based neural learners, instantiated in Fuzzy ARTMAP. The framework includes WB-Softmax, a white-box attack objective aligned with ARTMAP's category competition and map-field structure, and a final-model streaming evaluation principle requiring robustness to be assessed on the deployed streamed model rather than on stale intermediate states. Empirically, we show that WB-Softmax PGD provides a strong adaptive evaluator, achieving 89--100\% attack success on vanilla Fuzzy ARTMAP models across the evaluated benchmarks and consistently exceeding transfer and query-based baselines at matched budgets.

Second, we show that replay-free robustness is a protocol property, not only an attack-strength property. Offline adversarial training can appear effective under transfer evaluation yet collapse under adaptive white-box evaluation, while progressive two-stage selective training provides the strongest overall replay-free robustness across USPS, MNIST, Fashion-MNIST, and EMNIST-Letters.

Third, we show that ART's internal geometry provides more than post-hoc interpretability: it enables online diagnosis of structural and semantic reliability failures. Geometry monitoring reveals \emph{separation collapse}, namely the progressive loss of cross-class geometric separation caused by adversarial adaptation, while match-score analysis uncovers match-score inversion, showing that internal trust signals calibrated on vanilla models may become unreliable after adversarial adaptation.

The remainder of this paper is organized as follows. Section~2 reviews the Fuzzy ARTMAP background. Section~3 formalizes the threat model and evaluation protocol for single-pass streaming robustness. Section~4 presents the attack suite, including the proposed WB-Softmax adaptive white-box attack and complementary black-box transfer baselines. Section~5 introduces interpretable diagnostics and replay-free training rules for streaming ARTMAP, including geometry-based monitoring, match-score analysis, and separation-aware training. Section~6 describes the experimental setup. Section~7 reports the results and discussion. Section~8 concludes the paper and outlines future directions. The Appendix provides the constructive proof of Proposition~1 and derived unconditional robustness tables corresponding to the conditional clean-correct evaluation reported in the main text.


\section{Background}
\label{sec:background}

Fuzzy ART~\cite{carpenter1991fuzzy} extends Adaptive Resonance Theory to continuous-valued inputs. Given an input feature vector $\bm{x}\in[0,1]^d$, complement coding forms
\begin{equation}
\bm{I}(\bm{x})=[\bm{x};\,\bm{1}-\bm{x}] \in [0,1]^{2d},
\end{equation}
so that the coded input $\bm{I}(\bm{x})$ has constant $L^1$ norm:

\begin{equation}
|\bm{I}(\bm{x})| = \sum_{i=1}^d x_i + \sum_{i=1}^d (1-x_i)=d.
\label{eq:complement-norm}
\end{equation}
Here $|\bm{v}|=\|\bm{v}\|_1=\sum_i v_i$ for $\bm{v}\in[0,1]^m$. This normalization is important in ART because it reduces sensitivity to raw input magnitude and helps mitigate category proliferation caused by norm variation.

Each category $j$ is represented by a weight vector $\bm{w}_j\in[0,1]^{2d}$, which defines a hyperbox-like region in complement-coded space. For an input $\bm{I}$, Fuzzy ART computes the \emph{match function}
\begin{equation}
M_j(\bm{I})=\frac{|\bm{I}\wedge \bm{w}_j|}{|\bm{I}|},
\label{eq:match}
\end{equation}
and the \emph{choice function}
\begin{equation}
T_j(\bm{I})=\frac{|\bm{I}\wedge \bm{w}_j|}{\alpha+|\bm{w}_j|},
\label{eq:choice}
\end{equation}
where $\wedge$ denotes element-wise minimum and $\alpha>0$ is the choice parameter. Categories compete through winner-take-all selection:
\begin{equation}
J=\arg\max_j T_j(\bm{I}),
\label{eq:winner}
\end{equation}
and the winning category $J$ is accepted if
\begin{equation}
M_J(\bm{I})\ge \rho,
\label{eq:vigilance}
\end{equation}
where vigilance $\rho\in[0,1]$ controls category granularity, with larger $\rho$ producing finer partitions. If vigilance fails, mismatch reset inhibits the current winner and the search continues.

Once a category $J$ is accepted, the general Fuzzy ART learning rule is
\begin{equation}
\bm{w}_J^{\text{new}}
=
\beta(\bm{I}\wedge \bm{w}_J^{\text{old}})
+
(1-\beta)\bm{w}_J^{\text{old}},
\qquad \beta\in(0,1],
\label{eq:fuzzyart-learning}
\end{equation}
where $\beta$ is the learning-rate parameter. The fast-learning case used in this paper corresponds to $\beta=1$, for which \eqref{eq:fuzzyart-learning} reduces to
\begin{equation}
\bm{w}_J^{\text{new}}=\bm{I}\wedge \bm{w}_J^{\text{old}}.
\label{eq:fuzzyart-fast-learning}
\end{equation}

Fuzzy ARTMAP~\cite{carpenter1992fuzzy} extends this mechanism to supervised learning by coupling an input module $\mathrm{ART}_a$ with a label module $\mathrm{ART}_b$ through a map field $F^{ab}$ that links input categories to class labels. With map-field vigilance $\rho_{ab}=1.0$, each learned category is mapped to exactly one class label. When a prediction error occurs, \emph{match tracking} raises the $\mathrm{ART}_a$ vigilance parameter $\rho_a$ just enough to reject the current winning category and trigger a search for, or creation of, a more specific category associated with the correct label. This mechanism allows ARTMAP to learn incrementally while preserving stable category-to-label assignments.

These architectural properties make ARTMAP fundamentally different from standard deep classifiers in the context of adversarial robustness. Prediction is determined by winner-take-all competition among explicit categories, while learning proceeds through piecewise category updates in a strictly incremental, single-pass manner. Consequently, both the internal score structure and the effective decision boundaries evolve during training. Standard gradient-based white-box attacks are therefore not directly aligned with ARTMAP's prediction mechanism, and robustness evaluation must account not only for attack strength but also for boundary evolution induced by streaming updates. This motivates the dedicated threat-model and evaluation formulation introduced in Section~\ref{sec:threat}.

A further consequence of ARTMAP's explicit category structure is that internal scores are interpretable. In particular, the winning-category match value $M_J(\bm{I})$ can be viewed as a measure of input compatibility with learned categories, which naturally suggests rejection or abstention rules based on a match threshold. While this intuition is often reasonable for vanilla ARTMAP, adversarial training can reshape post-training match statistics and even change their ordering. This motivates the diagnostic and theoretical analysis developed in Section~\ref{sec:defense}.

\section{Threat Model and Evaluation Protocol}
\label{sec:threat}

This section formalizes the threat model and evaluation protocol studied throughout the paper. Unlike conventional adversarial-robustness settings, which typically assume offline multi-epoch optimization and evaluate a fixed trained model under adaptive or transfer attacks \cite{carlini2017towards,croce2021robustbench}, we consider a strict single-pass streaming regime in which the classifier evolves continuously during training and past samples are not revisited. In this setting, adversarial-example generation is not merely an implementation detail: it changes the effective attack distribution encountered during training and therefore changes what robustness claims actually mean.

\subsection{Streaming Setting and Threat Model}

We consider supervised Fuzzy ARTMAP trained in a strict single-pass incremental setting. Let
\begin{equation}
f_t:\mathcal{X}\rightarrow\mathcal{Y},\qquad t=0,1,\ldots,T,
\label{eq:model-sequence}
\end{equation}
denote the model state after processing $t$ training samples, where $f_0$ is the initial model and $f_T$ is the final deployed model after the single streamed pass. Here, $\mathcal{X}\subseteq[0,1]^d$ denotes the input space and $\mathcal{Y}$ denotes the discrete label space. Each sample is processed exactly once, without replay buffers, rehearsal, or repeated optimization over historical data. This distinguishes our setting from much of the continual-learning literature, where replay, memory, or episodic correction are standard tools \cite{kirkpatrick2017overcoming,parisi2019continual}. Because Fuzzy ARTMAP updates categories through match tracking, category creation, reset, and fast learning, both its internal category structure and its effective decision boundary can change substantially over time.

Our evaluation target is therefore the robustness of the final streamed model $f_T$ under attacks crafted against $f_T$ itself. This is a methodological choice rather than a convenience: in deployed incremental systems, the relevant adversary interacts with the currently deployed model after adaptation has occurred, not with an intermediate checkpoint or a stale pre-update snapshot.

We consider both white-box and black-box adversaries, following standard robustness-evaluation practice \cite{carlini2017towards,croce2021robustbench}, but adapted to the streaming setting. In the white-box setting, the adversary has full access to ARTMAP internals, including category weights $\{\bm{w}_j\}$, the map field $F^{ab}$, and model hyperparameters; attacks are crafted against the final trained model $f_T$, so that the attack objective is aligned with the deployed decision boundary. In the black-box setting, the adversary has no access to ARTMAP internals but does have access to the training set and its ground-truth labels, allowing labeled surrogate models to be trained for transfer attacks.

A central issue in streaming learners is whether adversarial examples used during training are generated against a fixed model snapshot or against the current evolving model state. To formalize this distinction, we let $\mathcal{A}(f,\bm{x},y)$ denote an attack procedure applied to model $f$ to a labeled sample $(\bm{x},y)$, where $\bm{x}\in\mathcal{X}$ is an input and $y\in\mathcal{Y}$ is its ground-truth class label.

Under \textbf{offline attack generation}, the adversarial example associated with $(\bm{x},y)$ is crafted against a fixed reference model $f_{\mathrm{ref}}$ (typically a clean-trained or pre-adversarial model):
\begin{equation}
\bm{x}_{\mathrm{adv}}^{\mathrm{off}}=\mathcal{A}(f_{\mathrm{ref}},\bm{x},y),
\label{eq:xadv-off}
\end{equation}
and then used throughout subsequent training:
\begin{equation}
f_t=U(f_{t-1};\,\bm{x}_{\mathrm{adv}}^{\mathrm{off}},y).
\label{eq:offline-update}
\end{equation}
Because $f_t$ changes with $t$, perturbations crafted against $f_{\mathrm{ref}}$ need not remain aligned with the final streamed model $f_T$. This mismatch is largely hidden in standard batch settings, where training repeatedly revisits the same data distribution and the evaluated model is often the same model family against which the adversarial examples were generated \cite{tramer2018ensemble}.

Under \textbf{online attack generation}, the adversarial example is regenerated against the current model state at each training step:
\begin{equation}
\bm{x}_{\mathrm{adv},t}=\mathcal{A}(f_{t-1},\bm{x}_t,y_t),
\label{eq:xadv-on}
\end{equation}
\begin{equation}
f_t=U(f_{t-1};\,\bm{x}_{\mathrm{adv},t},y_t).
\label{eq:online-update}
\end{equation}
This keeps attack generation aligned with the evolving boundary and avoids the staleness induced by a fixed reference model.

In this work, robustness is always evaluated using attacks crafted against the final trained model $f_T$, regardless of how adversarial training examples were generated. This choice prevents robustness from being overstated due to boundary-shift artifacts or stale perturbations and defines the evaluation object used throughout the remainder of the paper.

\textbf{Principle 1 (Offline Distribution Mismatch in Streaming Learners).}
In single-pass incremental learners without replay, adversarial examples crafted against a fixed model snapshot need not remain aligned with the final streamed model because the decision boundary evolves during training. Consequently, robustness observed under stale, transfer-based, or otherwise non-adaptive attacks does not imply robustness under adaptive white-box attacks on the final deployed model.

Principle 1 is not a standard theorem inherited from batch adversarial learning; rather, it is the central methodological principle of our streaming setting. It defines the evaluation mismatch studied throughout the paper and explains why offline adversarial training in streaming learners can appear effective under transfer-based black-box evaluation while failing under adaptive white-box attacks on the final model.

\subsection{Attack Budget and Robustness Metrics}

We consider untargeted $\ell_{\infty}$-bounded attacks in normalized pixel space $[0,1]^d$ with budgets
\begin{equation}
\epsilon\in\{0.05,0.10,\ldots,0.35\},
\label{eq:eps-grid}
\end{equation}
while enforcing $\bm{x}_{\mathrm{adv}}\in[0,1]^d$.

For gradient-based attacks, FGSM \cite{goodfellow2014explaining} applies the one-step update
\begin{equation}
\bm{x}_{\mathrm{adv}}
=
\mathrm{clip}_{[0,1]}\!\left(\bm{x}+\epsilon\,\mathrm{sign}(\nabla_{\bm{x}}L)\right),
\label{eq:fgsm}
\end{equation}
where $L$ denotes the attack loss. PGD \cite{madry2017towards} performs $K$ iterative steps with step size $\eta=\epsilon/4$, random initialization within the $\ell_{\infty}$ ball, and projection back to the feasible ball after each step. These attack definitions are standard; what is specific to our work is the streaming-aligned white-box objective introduced later in Section~4, together with the requirement that evaluation target the final streamed model.

We report four related performance measures. \emph{Clean accuracy} is the standard classification accuracy on unperturbed test inputs. \emph{Adversarial accuracy at budget $\epsilon$}, denoted $\mathrm{Acc}_{\mathrm{adv}}(\epsilon)$, is the classification accuracy measured on adversarially perturbed inputs constrained by that budget. In the main text, when we refer to \emph{robust accuracy} as a function of $\epsilon$, we mean this adversarial-accuracy curve $\mathrm{Acc}_{\mathrm{adv}}(\epsilon)$. We use the term ``robust'' only in this evaluation sense, not in the sense of a certified guarantee.

To summarize performance across budgets, we report the Area Under the Robust Accuracy Curve (AURAC):
\begin{equation}
\mathrm{AURAC}
=
\frac{1}{\epsilon_{\max}-\epsilon_{\min}}
\int_{\epsilon_{\min}}^{\epsilon_{\max}}
\mathrm{Acc}_{\mathrm{adv}}(\epsilon)\,d\epsilon,
\label{eq:aurac}
\end{equation}
approximated by trapezoidal integration on the grid in \eqref{eq:eps-grid}. AURAC is not a new metric; we use it because it summarizes adversarial performance across a range of perturbation strengths rather than at a single operating point. This is especially useful in streaming settings, where different training protocols may behave differently at low and high $\epsilon$.

\subsection{Sanity Checks and Validity Criteria}

Following Athalye et al.\ \cite{athalye2018obfuscated}, we apply four validity checks adapted to streaming ART models:
\begin{enumerate}
    \item single-step FGSM achieves non-trivial attack success;
    \item iterative PGD is at least as strong as FGSM;
    \item white-box attacks match or exceed black-box transfer attacks; and
    \item accuracy degrades smoothly as $\epsilon$ increases.
\end{enumerate}
These checks are not novel in themselves, but they are essential in our setting to verify that the proposed attack-and-evaluation pipeline is measuring genuine vulnerability rather than artifacts of non-smooth winner-take-all computation.

This section defines the robustness object studied in the remainder of the paper: not robustness to arbitrary stale perturbations, but robustness of the final streamed ARTMAP model under attacks aligned with both its prediction mechanism and its final post-training state.

\section{White-Box Attacks on Fuzzy ARTMAP}
\label{sec:attacks}

This section instantiates the generic attack loss $\mathcal{L}$ introduced in
Section~\ref{sec:threat} for adaptive white-box evaluation of Fuzzy ARTMAP.
The core difficulty is architectural: ARTMAP predicts through winner-take-all category competition followed by map-field assignment, while its computation contains non-differentiable selection and piecewise operations. As a result, standard gradient-based white-box attacks are not naturally aligned with the mechanism by which ARTMAP actually makes decisions and may therefore produce weak or misleading optimization signals.

Accordingly, the main contribution of this section is not a new outer optimization routine, but an ARTMAP-aligned attack objective. We propose WB-Softmax, a differentiable surrogate that converts category-level competition into class-level attack scores in a manner consistent with the ARTMAP map field. Rather than replacing the forward model with a generic smooth proxy, WB-Softmax preserves the original ARTMAP forward computation and introduces differentiability only in the loss used for optimization.

\paragraph{WB-Softmax: softmax-relaxed class loss}
ARTMAP selects the winning category by hard maximization of the choice values and then predicts through the map field. Because this hard winner-take-all step is non-differentiable almost everywhere, it blocks direct gradient-based optimization. Our key idea is therefore to replace hard winner selection in the attack objective with a differentiable softmax relaxation over category choice values, and then aggregate the resulting probability mass at the class level according to the map field.

Let $C_c$ denote the set of categories mapped to class $c$. We define the class-level probability of class $c$ by summing the softmax mass of all categories assigned to that class:
\begin{equation}
p_c
=
\sum_{j\in C_c}
\frac{\exp(T_j/\tau)}
{\sum_k \exp(T_k/\tau)},
\label{eq:wbsoftmax}
\end{equation}
where $\tau>0$ is the softmax temperature. The WB-Softmax attack loss is then the negative log-likelihood of the true class:
\begin{equation}
L_{\mathrm{WB\text{-}Softmax}}
=
-\log(p_y).
\label{eq:wb-loss}
\end{equation}

This construction is specific to ARTMAP. The attack is driven by category-level
choice values and aggregated according to the class assignments encoded by the
map field, so the objective remains aligned with ARTMAP's actual
competition-and-mapping structure rather than with an external differentiable
proxy. We use probability-mass aggregation (sum over categories) rather than a
hard classwise max so that gradients remain informative even when multiple
same-class categories compete near the decision boundary.

We set $\tau=0.01$ based on attack-strength ablations. Lower temperatures
concentrate probability mass on near-winning categories while still preserving
useful gradient flow. Maximizing~\eqref{eq:wbsoftmax} therefore pushes
probability mass away from the true class and promotes misclassification.

\paragraph{Smooth surrogate construction under complement coding}
WB-Softmax introduces differentiability only in the aggregation step used to
construct the attack objective; the forward ARTMAP computation itself remains
unchanged. This differs from broader BPDA-style gradient substitutions that
modify the forward/backward interface more aggressively~\cite{athalye2018obfuscated}.
Our goal is not to replace ARTMAP with a generic smooth model, but to build a
white-box objective that respects its native category structure.

Under complement coding,
\begin{equation}
\bm{I}(\bm{x}) = [\bm{x};\,\bm{1}-\bm{x}],
\label{eq:comp-coded-attack}
\end{equation}
let $\bm{I}_{1:d}$ denote the first $d$ components of $\bm{I}(\bm{x})$ and let $\bm{I}_{d+1:2d}$ denote the last $d$ components. By construction,
\begin{equation}
\bm{I}_{1:d}=\bm{x},
\qquad
\bm{I}_{d+1:2d}=\bm{1}-\bm{x}.
\label{eq:comp-coded-blocks}
\end{equation}
The chain rule therefore gives
\begin{equation}
\nabla_{\bm{x}}L
=
\nabla_{\bm{I}_{1:d}}L
-
\nabla_{\bm{I}_{d+1:2d}}L.
\label{eq:chain}
\end{equation}
Equation~\eqref{eq:chain} ensures that gradients are propagated correctly through the coded representation: perturbations in $\bm{x}$ induce equal-magnitude, opposite-signed changes in the two halves of $\bm{I}(\bm{x})$. Thus, WB-Softmax preserves the original complement-coded forward representation while still enabling gradient-based white-box optimization.

\paragraph{WB-Softmax PGD attack}
We instantiate WB-Softmax inside a standard projected gradient descent (PGD)
outer loop with random start inside the $\ell_\infty$ ball. The novelty is
therefore the ARTMAP-aligned objective, not the PGD shell itself.
Algorithm~\ref{alg:wb-softmax} summarizes the complete attack procedure used
throughout this paper. We use $\mathcal{U}(-\epsilon,\epsilon)$ to denote
element-wise uniform noise in $[-\epsilon,\epsilon]$, and
$\Pi_{[\bm{x}-\epsilon,\bm{x}+\epsilon]}(\cdot)$ to denote projection onto the
$\ell_\infty$ ball around $\bm{x}$ via coordinate-wise clipping.

\begin{algorithm}[t]
\caption{WB-Softmax PGD Attack}
\label{alg:wb-softmax}
\small
\begin{algorithmic}[1]
\REQUIRE input $\bm{x}\in[0,1]^d$, ground-truth label $y\in\mathcal{Y}$, perturbation budget $\epsilon$, number of steps $K$, step size $\eta=\epsilon/4$, temperature $\tau$
\STATE $\bm{x}_{\mathrm{adv}} \leftarrow \mathrm{clip}_{[0,1]}\!\big(\bm{x} + \mathcal{U}(-\epsilon,\epsilon)\big)$
\FOR{$k=1$ \TO $K$}
  \STATE $\bm{I}\leftarrow[\bm{x}_{\mathrm{adv}};\,\bm{1}-\bm{x}_{\mathrm{adv}}]$
  \STATE $T_j \leftarrow \dfrac{\lvert \bm{I}\wedge \bm{w}_j\rvert}{\alpha+\lvert \bm{w}_j\rvert}\qquad \forall j$
  \STATE $q_j \leftarrow \dfrac{\exp(T_j/\tau)}{\sum_m \exp(T_m/\tau)}\qquad \forall j$
  \STATE $p_c \leftarrow \sum_{j\in\mathcal{C}_c} q_j \qquad \forall c$ \hfill $\triangleright$ class-level probabilities
  \STATE $\mathcal{L}\leftarrow -\log(p_y)$
  \STATE $\bm{g}\leftarrow \nabla_{\bm{x}_{\mathrm{adv}}}\mathcal{L}$ \hfill $\triangleright$ autograd (equiv.\ to Eq.~\eqref{eq:chain})
  \STATE $\bm{x}_{\mathrm{adv}} \leftarrow \Pi_{[\bm{x}-\epsilon,\bm{x}+\epsilon]}\!\big(\bm{x}_{\mathrm{adv}}+\eta\,\mathrm{sign}(\bm{g})\big)$
  \STATE $\bm{x}_{\mathrm{adv}} \leftarrow \mathrm{clip}_{[0,1]}(\bm{x}_{\mathrm{adv}})$
\ENDFOR
\RETURN $\bm{x}_{\mathrm{adv}}$
\end{algorithmic}
\end{algorithm}

WB-Softmax operationalizes the evaluation principle of
Section~\ref{sec:threat}: if robustness is to be assessed against adaptive
white-box attacks on the final streamed model, then the attack objective itself
must be aligned with the mechanism by which that model actually predicts. More
broadly, this illustrates a general lesson for non-smooth winner-take-all
learners: meaningful white-box evaluation requires a surrogate objective aligned
with the model's native decision rule, rather than with an unrelated
differentiable proxy.

\paragraph{Black-box transfer attacks}
For black-box evaluation, we train surrogate classifiers on the same training set and generate transfer attacks with PGD. Our primary surrogate is a SimpleCNN, and we additionally consider an LRS-regularized surrogate to strengthen transferability. These surrogates are used only to generate black-box transfer attacks; they do not alter the ARTMAP training or inference mechanism. Full surrogate architectures and training hyperparameters are deferred to Section~6.


\section{Interpretable Diagnostics and Reliability-Aware Training Rules}
\label{sec:defense}
Throughout this section, we use \emph{absorption} to mean that an incoming sample is assigned to an existing accepted category and updates that category under the fast-learning rule, rather than creating a new category.

ART exposes its internal state through explicit category geometry, enabling a diagnosis-to-rule workflow that is difficult to realize in black-box models. In our setting, this structure is not merely descriptive: it can be monitored online to detect structural failure modes that emerge during adversarial training, and it can be used to design lightweight replay-free interventions that act directly on category creation and fast-learning updates of existing categories. This makes ARTMAP unusual among streaming learners: the same structure that supports incremental learning also supports online interpretability and targeted robustness interventions without replay.

This capability is also relevant to deployment reliability. In streaming systems, internal quantities such as match may be reused for rejection, abstention, or reliability filtering~\cite{chow1970reject,geifman2017selective,geifman2019selectivenet,hendrycks2017baseline}. If adversarial training changes the relationship between such quantities and correctness, then trust signals calibrated on vanilla models may become unsafe after adaptation. Accordingly, this section studies two coupled questions: (i) how explicit category geometry can diagnose robustness failure modes online, and (ii) whether post-training internal scores, especially match-based scores, remain reliable proxies for correctness. We show that both questions admit actionable answers in ARTMAP: geometry monitoring exposes \emph{separation collapse}, a failure mode in which different-class categories become increasingly overlapping during adversarial adaptation, while match-score analysis reveals a reliability failure that can invalidate vanilla-calibrated rejection rules.

\subsection{Online Structural Diagnostics via iCVIs and Geometry Indicators}

Cluster validity indices (CVIs) play a central role in assessing clustering quality, and incremental cluster validity indices (iCVIs) provide online counterparts that can be updated recursively as samples arrive, including online iCVI formulations for streaming clustering~\cite{moshtaghi2019online,dasilva2020icvi,dasilva2023icviartmap}. In black-box deep models,
related geometric analyses typically require post-hoc embeddings or external
probes. In contrast, ART categories have explicit geometric representations, so class-conditional separation, overlap, and compactness can be computed directly from the current category set during training.

In this work, we use online structural diagnostics rather than post-hoc
visualization. We monitor two geometry indicators that are especially
informative during adversarial training: (i) \emph{minimum separation}, the
smallest non-overlap margin between different-class categories, and
(ii) \emph{overlap risk}, the maximum normalized intersection between
wrong-class category pairs. Because these indicators depend only on the current category weights $\{\bm{w}_j\}$, they can be updated online without storing past samples. These indicators are particularly useful for revealing failure modes that are not apparent from accuracy curves or category counts alone, and they provide actionable signals for designing targeted replay-free training rules. Section~\ref{sec:results} presents a case study illustrating these diagnostics on one of the benchmarks introduced in Section~\ref{sec:setup}.

\noindent\textbf{Geometry indicators.}
For the purpose of online overlap diagnostics, we use a diagnostic box representation in complement-coded space induced by the category weights. Let category $j$ be represented by lower and upper bounds $\bm{\ell}_j=\bm{w}_j$ and $\bm{u}_j=\bm{1}$, and let $\mathcal{C}(j)$ denote its mapped class. Here $\bm{\ell}_j,\bm{u}_j\in[0,1]^{2d}$. For two boxes $j$ and $k$, define the $L^1$ intersection length
\begin{equation}
\mathrm{Int}(j,k)
=
\big\lvert
\max(\bm{0},\,\min(\bm{u}_j,\bm{u}_k)-\max(\bm{\ell}_j,\bm{\ell}_k))
\big\rvert_1,
\label{eq:box-intersection}
\end{equation}
where the $\min(\cdot,\cdot)$ and $\max(\cdot,\cdot)$ operators inside \eqref{eq:box-intersection} act element-wise before the final $L^1$ aggregation. We then define the normalized overlap as
\begin{equation}
\mathrm{Ov}(j,k)
=
\frac{\mathrm{Int}(j,k)}
{\min(\lvert \bm{u}_j-\bm{\ell}_j\rvert_1,\;\lvert \bm{u}_k-\bm{\ell}_k\rvert_1)+\varepsilon_0},
\label{eq:overlap}
\end{equation}
where $\varepsilon_0$ avoids division by zero. We further define \emph{overlap risk} as
\begin{equation}
\mathrm{OR}
=
\max_{\mathcal{C}(j)\neq \mathcal{C}(k)} \mathrm{Ov}(j,k),
\label{eq:overlap-risk}
\end{equation}
and \emph{minimum separation} as the smallest non-overlap margin between different-class boxes,
\begin{equation}
\mathrm{MS}
=
\min_{\mathcal{C}(j)\neq \mathcal{C}(k)}
\Big[\,1-\mathrm{Ov}(j,k)\,\Big].
\label{eq:min-sep}
\end{equation}
These quantities depend only on the current category set $\{\bm{w}_j\}$ and can therefore be maintained online throughout training.

\subsection{Streaming Training Protocols, Failure Modes, and Targeted Interventions}
We revisit the training-protocol taxonomy of Section~\ref{sec:threat} from a
geometric and reliability perspective. We compare replay-free adversarial
training along two axes: \emph{when} adversarial examples are generated
(offline versus online) and \emph{which} generated examples are used for
updating (standard versus selective). Offline variants use perturbations
crafted against a fixed clean-trained reference model, whereas online variants
regenerate perturbations against the current streamed model state. Standard variants update on all generated adversarial examples, whereas selective variants update only on adversarial examples that induce misclassification, that is, those satisfying $f(\bm{x}_{\mathrm{adv}})\neq y$. Within standard updates, the absorb-vs-create decision can additionally be gated by predicted hyperbox overlap (Algorithm~\ref{alg:sep-aware}), without changing whether a sample triggers an update. Algorithm~\ref{alg:training-protocols} summarizes these protocol choices.

These protocol differences are consequential because adversarial updates in
ARTMAP can alter category structure rather than merely adjust a smooth
boundary. Depending on when perturbations are generated and which samples are
admitted for updating, adversarial training may trigger mismatch reset,
new-category formation, or cross-class encroachment. Consequently, robustness
in ARTMAP depends not only on attack strength, but also on how adversarial
samples are scheduled and filtered during streaming adaptation.

A practical consequence is category proliferation. Adversarial training can
substantially increase the number of categories, often by roughly a factor of
two, thereby increasing memory footprint and evaluation cost. When an
adversarial example $\bm{x}_{\mathrm{adv}}$ falls outside existing hyperboxes, learning may
create a new category instead of absorbing the sample into an existing one.
Selective updating mitigates this unnecessary growth by concentrating updates on
adversarial examples that reveal genuine model failures.

To improve stability across perturbation strengths, we employ progressive
two-stage selective training. Stage~1 samples $\epsilon\in[0.05,0.15]$, and
Stage~2 extends to $\epsilon\in[0.15,0.35]$, using online selective filtering
in both stages. This schedule separates adaptation to moderate perturbations
from later adaptation to stronger attacks, reducing the instability that can
arise when large-$\epsilon$ perturbations are introduced too early and
providing a more controlled path for category evolution.

\begin{algorithm}[t]
\caption{Adversarial Training Protocols}
\label{alg:training-protocols}
\small
\begin{algorithmic}[1]
\REQUIRE training stream $\{(\bm{x}_i,y_i)\}$, attack mode $\in\{\textsc{offline},\textsc{online}\}$, update rule $\in\{\textsc{standard},\textsc{selective}\}$
\IF{attack mode = \textsc{offline}}
    \STATE Train on all clean samples $\{(\bm{x}_i,y_i)\}$ \hfill $\triangleright$ Pass 1
    \FOR{each $(\bm{x}_i,y_i)$}
        \STATE Generate $\bm{x}_{\mathrm{adv},i}$ against the fixed clean-trained model
        \IF{update rule = \textsc{standard} \OR $f(\bm{x}_{\mathrm{adv},i})\neq y_i$}
            \STATE Train on $(\bm{x}_{\mathrm{adv},i},y_i)$
        \ENDIF
    \ENDFOR
\ELSE
    \FOR{each $(\bm{x}_i,y_i)$}
        \STATE Train on $(\bm{x}_i,y_i)$
        \STATE Generate $\bm{x}_{\mathrm{adv},i}$ against the current model
        \IF{update rule = \textsc{standard} \OR $f(\bm{x}_{\mathrm{adv},i})\neq y_i$}
            \STATE Train on $(\bm{x}_{\mathrm{adv},i},y_i)$
        \ENDIF
    \ENDFOR
\ENDIF
\end{algorithmic}
\end{algorithm}

These observations suggest the following failure-mode taxonomy for streaming
ARTMAP under adversarial training.

\noindent\textbf{Failure-mode taxonomy (streaming ARTMAP under adversarial training).}
We observe two recurring pathology classes that are detectable online from the
current category set:
\begin{itemize}[leftmargin=*,topsep=2pt,itemsep=2pt]
\item \textbf{Separation collapse:}
\emph{Detection}---minimum separation drops sharply while compactness remains
high (Table~\ref{tab:separation-collapse}).
\emph{Risk}---high-$\epsilon$ robustness degrades because cross-class overlap increases during adversarial updates to existing categories.
\emph{Mitigation}---an overlap-based separation-aware update rule (Algorithm~3) that checks predicted overlap before fast-learning absorption (Section~\ref{sec:results} reports its empirical behavior).

\item \textbf{Match-score inversion:}
\emph{Detection}---after selective adversarial training, adversarial samples
can attain higher match than clean samples in some regimes.
\emph{Risk}---match-threshold rejection calibrated on vanilla models may fail
or invert.
\emph{Mitigation}---calibrate rejection jointly with the chosen training
protocol; do not transfer thresholds across variants.
\end{itemize}

The second failure mode is semantic rather than purely geometric: even when
match remains locally well-defined, its interpretation as a reliability signal
can break after selective adversarial training.

\begin{lemma}[Match Preservation under Selective Absorption]
\label{lem:match-preservation}
Consider selective adversarial training for streaming Fuzzy ARTMAP, and let an
adversarial sample $\bm{x}_{\mathrm{adv}}$ be absorbed into an existing
correct-class category $j$ under fast learning. Then the category-level match
of that same sample to category $j$ is preserved after the update. Consequently,
selective absorption does not reduce the absorbed-category match of updated
adversarial samples, even though correctness is determined by global category
competition and therefore need not remain monotonic in match.
\end{lemma}

Lemma~\ref{lem:match-preservation} isolates the local effect of selective
absorption: for an adversarial sample that is actually used for updating, fast
learning preserves its match to the absorbed category rather than decreasing it.
This local preservation does not by itself imply correct prediction, because
final decisions depend on winner-take-all competition across all categories and
classes. More broadly, selective updates, category creation, and competition
shifts can reshape the post-training match statistics of adversarial inputs,
which motivates the stronger non-monotonicity statement in
Proposition~\ref{prop:match-nonmonotone}.

\begin{proposition}[Non-monotonicity of Match as a Correctness Proxy after Selective Training]
\label{prop:match-nonmonotone}
Consider a streaming Fuzzy ARTMAP trained without replay using selective
adversarial training, where only misclassified adversarial samples
$(\bm{x}_{\mathrm{adv}},y)$ satisfying $f(\bm{x}_{\mathrm{adv}})\neq y$
trigger updates. Let $M_j(\bm{I}(\bm{x}))$ denote the post-training
category-level match score of input $\bm{x}$ with respect to category $j$.
Then, after selective adversarial training, match need not be a monotone
indicator of correctness; in particular, there exist labeled inputs
$(\bm{x}_1,y_1)$ and $(\bm{x}_2,y_2)$ and corresponding categories $j_1$ and
$j_2$ such that
\begin{equation}
\begin{aligned}
M_{j_1}(\bm{I}(\bm{x}_1)) &> M_{j_2}(\bm{I}(\bm{x}_2)), \\
f(\bm{x}_1) &\neq y_1, \\
f(\bm{x}_2) &= y_2.
\end{aligned}
\label{eq:supp-prop-nonmonotone}
\end{equation}
Thus, a larger post-training match score does not necessarily imply a higher
likelihood of correct classification. Consequently, match-threshold rejection
rules calibrated on vanilla models may fail after selective adversarial
training.
\end{proposition}

A constructive proof is provided in Appendix A.

Proposition~\ref{prop:match-nonmonotone} explains why a fixed match-threshold
detector that performs well on vanilla models can fail after selective
adversarial training: the training process can reshape adversarial match
statistics without preserving a monotone relationship between match and
correctness. Consequently, thresholds calibrated on vanilla models should not
be transferred across training protocols without validation, since AUC can drop
sharply and may even fall below chance. Importantly, the proposition does not
assert that selective absorption strictly increases match in the
absorbed-category case. Instead, it makes the more conservative claim that,
after selective adversarial training, post-training match need not remain a
monotone proxy for correctness. This result matters beyond adversarial evaluation: in online deployment, any rejection, abstention, or human-escalation mechanism that relies on match as a trust signal must be revalidated after adversarial training. In other words, the issue is not only robustness of prediction, but also robustness of the internal confidence proxy used for downstream decision support.

\emph{Separation-aware training rule.}
Motivated by separation collapse, we propose separation-aware training, which modifies the \emph{form} of the adversarial update (absorption vs.\ new-category creation), not whether one occurs. For each adversarial example $\bm{x}_{\mathrm{adv}}$, we (i) identify the best-matching correct-class category $j$, (ii) simulate the post-update hyperbox
$\bm{w}_j^{\mathrm{new}}=\bm{w}_j \wedge \bm{I}(\bm{x}_{\mathrm{adv}})$, and (iii) evaluate its overlap with wrong-class categories. If the predicted overlap exceeds a threshold $\theta$, we create a new category instead of absorbing $\bm{x}_{\mathrm{adv}}$ into the existing one, that is, instead of updating the accepted category under fast learning. Overlap is computed as the normalized $L^1$ intersection of hyperbox bounds, directly penalizing expansions that intrude into wrong-class regions. This rule preserves class separation while still incorporating informative adversarial samples.

\begin{algorithm}[t]
\caption{Separation-Aware Update Rule for Adversarial Samples}
\label{alg:sep-aware}
\small
\begin{algorithmic}[1]
\REQUIRE adversarial sample $\bm{x}_{\mathrm{adv}}$, ground-truth label $y$, current categories $\{\bm{w}_j\}$, class mapping $\mathcal{C}(j)$, threshold $\theta$
\STATE $\bm{I}\leftarrow[\bm{x}_{\mathrm{adv}};\,\bm{1}-\bm{x}_{\mathrm{adv}}]$
\STATE $j^\star \leftarrow \arg\max_{j:\mathcal{C}(j)=y} T_j(\bm{I})$ \hfill $\triangleright$ highest-choice correct-class category
\STATE $\bm{w}^{\mathrm{new}} \leftarrow \bm{w}_{j^\star}\wedge \bm{I}$ \hfill $\triangleright$ simulated fast-learning update
\STATE $\Delta \leftarrow \max_{k:\mathcal{C}(k)\neq y}\; \mathrm{Ov}\big((\bm{\ell}^{\mathrm{new}},\bm{u}^{\mathrm{new}}),(\bm{\ell}_k,\bm{u}_k)\big)$
\STATE \hspace{1.45em} where $\bm{\ell}^{\mathrm{new}}=\bm{w}^{\mathrm{new}},\ \bm{u}^{\mathrm{new}}=\bm{1},\ \bm{\ell}_k=\bm{w}_k,\ \bm{u}_k=\bm{1}$
\IF{$\Delta > \theta$}
  \STATE Create a new category for class $y$ with $\bm{w}\leftarrow \bm{I}$
\ELSE
  \STATE Update the existing category: $\bm{w}_{j^\star}\leftarrow \bm{w}^{\mathrm{new}}$
\ENDIF
\end{algorithmic}
\end{algorithm}

Match-score statistics highlight a caveat for rejection-based defenses. On
vanilla models, adversarial examples tend to produce lower match scores than
clean samples, suggesting that a simple threshold on match could reject many
adversarial inputs. However, after selective adversarial training this
relationship can reverse: adversarial examples can attain higher match scores
than clean samples, and the corresponding match-threshold detector can
degrade below chance. Section~\ref{sec:results} reports the empirical AUC
collapse on a representative dataset. Practically, rejection thresholds must
be calibrated for the specific trained variant and validated jointly with the
chosen training protocol, rather than transferred from a vanilla
model~\cite{chow1970reject,geifman2017selective,geifman2019selectivenet,hendrycks2017baseline}.
We compute this AUC by using the match score as a scalar detector score and
sweeping a single threshold. This reversal is consistent with selective
training preserving absorbed-category match while reshaping post-training
attained-match statistics through selective updates, category creation, and
winner competition. More broadly, it shows that interpretable internal scores
in streaming models should not be assumed to remain trustworthy after
adversarial training merely because they were reliable in the vanilla regime.
\section{Experimental Setup}
\label{sec:setup}

We evaluate on four image-classification benchmarks: USPS (7{,}291 train / 2{,}007 test, $16\times16$), MNIST (60{,}000 / 10{,}000, $28\times28$), Fashion-MNIST (60{,}000 / 10{,}000, $28\times28$), and EMNIST-Letters (124{,}800 / 20{,}800, 26 classes, $28\times28$). These benchmarks are not intended to define a large-scale vision leaderboard; rather, they serve as controlled streaming testbeds in which ARTMAP category growth, final-model attack alignment, and replay-free adversarial-training protocols can be compared systematically across different input dimensions and class structures. All images are normalized to $[0,1]$ and flattened before training. Unless otherwise noted, we use Fuzzy ARTMAP with $\alpha=10^{-3}$, $\beta=1.0$, and $\rho_{ab}=1.0$.

For vanilla models, we sweep the input-module vigilance parameter $\rho_a\in\{0.0,0.1,\ldots,0.9\}$ to characterize the clean-accuracy/robustness tradeoff. Figure~1 reports representative vulnerability results for $\rho_a=0.5$--$0.9$, since lower-vigilance settings yield poor clean accuracy and are not competitive in practice. All defense methods are evaluated at $\rho_a=0.9$, which provides the strongest overall robustness in our setting.

For white-box attacks, we use WB-Softmax with temperature $\tau = 0.01$ and perturbation budgets
$\epsilon \in \{0.05, 0.10, \ldots, 0.35\}$.
We verified the WB-Softmax hyperparameters through attack-strength ablations: $\tau$ was swept over the range $[0.005,0.1]$ with evaluation at $\{0.005,0.01,0.02,0.05,0.10\}$, and PGD steps were varied over $\{1,5,10,20\}$. Across datasets, $\tau=0.01$ and PGD-20 consistently produced the strongest and most stable attacks, and we use this configuration throughout the main evaluation.

For black-box transfer robustness, we use two surrogate families. The first is a SimpleCNN trained separately on each dataset. It contains two convolutional blocks, each composed of convolution, nonlinear activation, and spatial downsampling, followed by a fully connected classifier that outputs class logits. The second is an LRS-regularized surrogate~\cite{wu2024lrs}, trained for 10 epochs using Adam ($\mathrm{lr}=10^{-3}$) with gradient-norm regularization $\lambda=2500$. We selected $\lambda=2500$ from a brief empirical sweep $\lambda\in\{500,1000,2000,2500\}$ and observed that transfer-attack strength, measured by AURAC, is relatively insensitive within this range.

We compare seven training variants (Table~\ref{tab:main-results}). The first three are non-selective (every adversarial triggers an update); the latter four are selective (only misclassifying adversarials trigger updates) or progressive selective:
\begin{itemize}
    \item \textbf{Vanilla}: no-defense baseline.
    \item \textbf{AdvTrain (off)}: offline adversarial training that updates on all generated adversarial examples.
    \item \textbf{AdvTrain (on)}: online adversarial training that regenerates adversarial examples during streaming updates and updates on all of them.
    \item \textbf{Sep-Aware}: online adversarial training (every adversarial triggers an update) where an overlap-gated decision rule chooses between fast-learning absorption into the best-matching correct-class category and creation of a new category (Algorithm~3, threshold $\theta$).
    \item \textbf{Selective (off)}: offline adversarial training with updates restricted to adversarial examples that induce misclassification.
    \item \textbf{Selective (on)}: online adversarial training with the same selective-update rule.
    \item \textbf{Two-Stage Sel.}: progressive two-stage selective training with moderate-$\epsilon$ adaptation in Stage~1 and stronger-$\epsilon$ adaptation in Stage~2.
\end{itemize}

For online variants, adversarial examples are generated on-the-fly with WB-Softmax and discarded immediately after use, preserving the strict single-pass streaming regime. Separation-aware training is motivated by the separation-collapse pattern discussed in Section~\ref{sec:defense}; Tables~\ref{tab:main-results} and~\ref{tab:bb-results} report its cross-dataset behavior, and Section~\ref{sec:results} reports the threshold ablation.

To keep adaptive PGD-20 evaluation tractable for large-category models (up to 228K categories), we evaluate adversarial accuracy on a fixed subset of $n=1000$ test samples that are correctly classified at $\epsilon=0$. Unless otherwise noted, the reported robust-accuracy curves and AURAC values are therefore conditional on correct clean prediction.
In figures, we abbreviate this as ``Cond.\ robust accuracy''. This conditional evaluation avoids conflating adversarial failure with clean misclassification and supports consistent comparison across defense methods with different clean accuracies. Clean accuracy is reported separately in all main defense tables so that robustness--utility tradeoffs remain visible.

White-box robustness is evaluated under WB-Softmax PGD, our strongest differentiable attack. Black-box robustness is evaluated primarily under CNN-transfer PGD and LRS-transfer PGD; Square attack~\cite{andriushchenko2020square} is used as an additional black-box baseline for vanilla models (Table~\ref{tab:attack-strength}). All adversarial training methods also use WB-Softmax attacks during training so that the defenses are evaluated against the strongest available adaptive white-box attacker.

Code will be released upon publication. Additional theoretical details, including the constructive proof of Proposition~1 and auxiliary robustness summaries, are provided in Appendix A and Appendix B.

\section{Results and Discussion}
\label{sec:results}

This section reports four main findings. First, vanilla Fuzzy ARTMAP is highly vulnerable under an adaptive white-box threat once the attack objective is aligned with ARTMAP's native decision rule. Second, defense rankings depend strongly on evaluation protocol: offline adversarial training can appear strong under transfer attacks yet collapse under adaptive white-box evaluation, directly supporting Principle~1. Third, among replay-free defenses, progressive two-stage selective training provides the strongest overall robustness across datasets. Separation-aware training, the geometric intervention motivated by separation collapse, matches standard online adversarial training in robustness; the overlap constraint does not fire in the recommended operating regime, and we report its behavior across $\theta$ later in this section. Fourth, post-training match statistics can become semantically unreliable after selective adversarial training, indicating that rejection rules calibrated on vanilla models should not be transferred across training protocols without revalidation.

\subsection{Baseline Vulnerability}

Figure~\ref{fig:vulnerability} characterizes the vulnerability of vanilla Fuzzy ARTMAP across vigilance levels under both adaptive white-box WB-Softmax PGD and black-box transfer PGD. Across all four datasets, WB-Softmax is consistently stronger than black-box transfer, showing that once the attack objective is aligned with ARTMAP's category competition and map-field structure, white-box optimization yields substantial attack success. This result is methodologically important: it confirms that ART models are not inherently protected by non-smooth winner-take-all structure, and that adaptive white-box evaluation is both feasible and necessary.

Higher vigilance improves robustness under both attack types, consistent with the interpretation that finer category structure can reduce vulnerability to broader perturbation regions. At the same time, the gap between WB-Softmax and transfer attacks is especially informative: black-box transfer alone would materially underestimate the vulnerability of the deployed ARTMAP model. This observation motivates the remainder of the section, where all defenses are compared under the stronger adaptive white-box protocol.

\begin{figure}[t]
\centering
\includegraphics[width=0.99\textwidth]{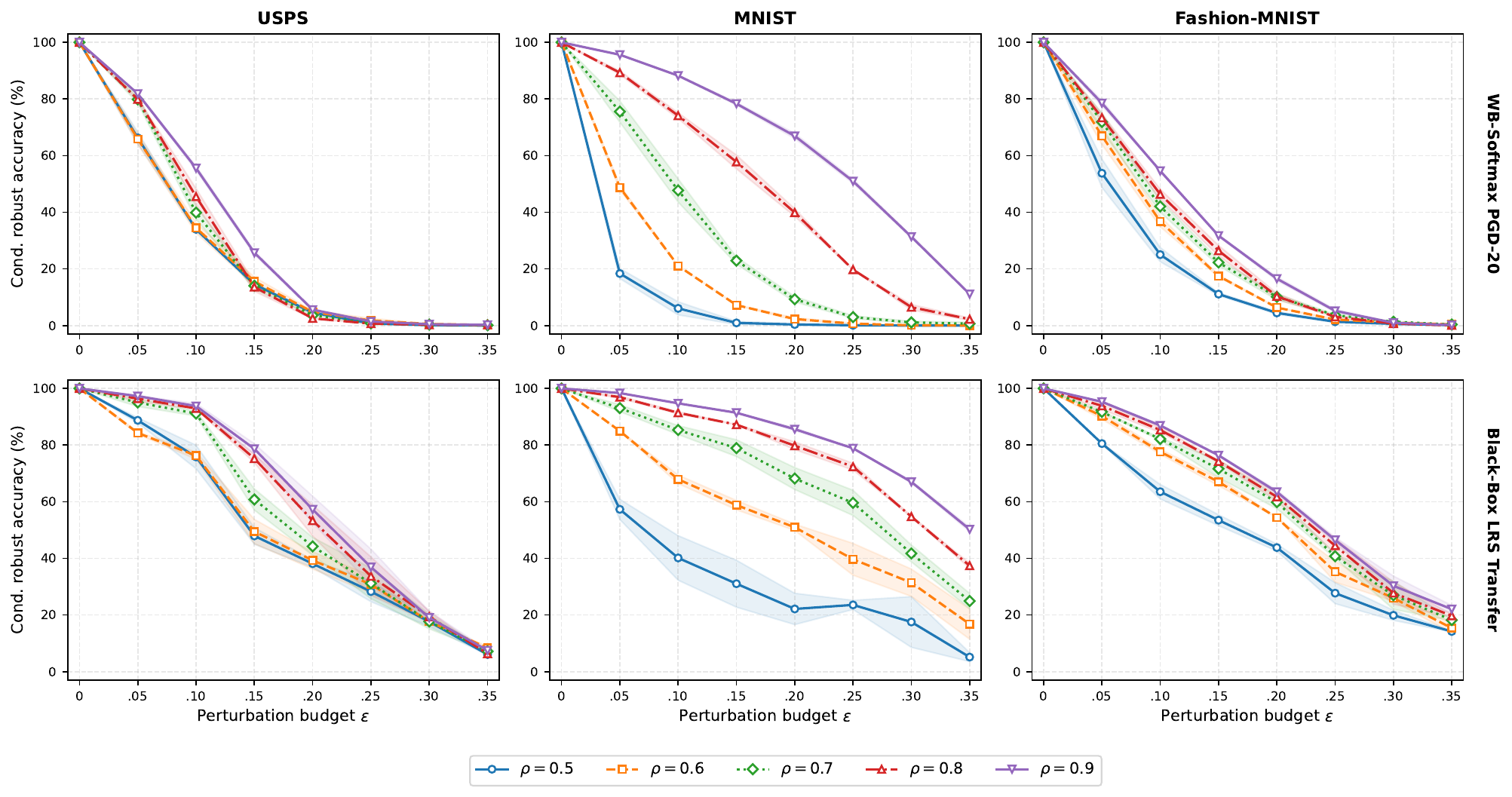}
\caption{Baseline vulnerability of vanilla Fuzzy ARTMAP ($\epsilon \leq 0.35$). \textbf{Top}: White-box Softmax PGD (WB-Softmax) attack with temperature $\tau=0.01$. \textbf{Bottom}: Black-box transfer PGD attack. The y-axis label ``Cond. robust accuracy'' denotes conditional robust accuracy, i.e., accuracy measured on the clean-correct subset.}
\label{fig:vulnerability}
\end{figure}

\subsection{Defense Comparison Under WB-Softmax}

Table~\ref{tab:main-results} compares all defenses at $\rho_a=0.9$ under the adaptive white-box WB-Softmax threat model, while Figure~\ref{fig:defense-comparison} shows representative accuracy--$\epsilon$ curves on Fashion-MNIST. We report clean accuracy, adversarial accuracy at $\epsilon=0.30$, AURAC, and category count. Separation-aware training is included for completeness across datasets, while Table~\ref{tab:usps-case-study} highlights its targeted high-$\epsilon$ effect on USPS.

Tables~\ref{tab:main-results} and~\ref{tab:bb-results} should be read jointly. Table~\ref{tab:main-results} reports robustness under the adaptive white-box WB-Softmax threat model, whereas Table~\ref{tab:bb-results} reports transfer-based black-box robustness. The resulting defense rankings differ sharply, and that reversal is itself one of the main findings of the paper.

\begin{table}[t]
\centering
\caption{Defense Comparison under WB-Softmax PGD-20 ($\rho = 0.9$). Mean$\pm$std over 3 seeds}
\label{tab:main-results}
\setlength{\tabcolsep}{2pt}
\footnotesize
\begin{tabular}{@{}l *{8}{c}@{}}
\toprule
& \multicolumn{4}{c}{USPS} & \multicolumn{4}{c}{MNIST} \\
\cmidrule(lr){2-5} \cmidrule(lr){6-9}
Method & \begin{tabular}[c]{@{}c@{}}Clean\\Acc.\end{tabular} & \begin{tabular}[c]{@{}c@{}}Acc$_{\mathrm{adv}}$\\(0.30)\end{tabular} & AURAC & \begin{tabular}[c]{@{}c@{}}\#\\Cat.\end{tabular} & \begin{tabular}[c]{@{}c@{}}Clean\\Acc.\end{tabular} & \begin{tabular}[c]{@{}c@{}}Acc$_{\mathrm{adv}}$\\(0.30)\end{tabular} & AURAC & \begin{tabular}[c]{@{}c@{}}\#\\Cat.\end{tabular} \\
\midrule
Vanilla & 92.3$\pm$0.6 & 0.5$\pm$0.1 & 21.7$\pm$0.2 & 6K & 94.9$\pm$0.2 & 31.7$\pm$0.6 & 61.5$\pm$0.1 & 56K \\
\midrule
AdvTrain (off) & 92.4$\pm$0.1 & 0.3$\pm$0.1 & 12.6$\pm$0.4 & 12K & 94.5$\pm$0.1 & 0.0$\pm$0.0 & 14.8$\pm$0.2 & 103K \\
AdvTrain (on) & 92.9$\pm$0.4 & \textbf{10.6$\pm$1.2} & 25.9$\pm$0.6 & 13K & 94.1$\pm$0.1 & 21.1$\pm$4.4 & 51.2$\pm$0.6 & 111K \\
\midrule
Selective (off) & 92.2$\pm$0.0 & 3.0$\pm$0.9 & 25.4$\pm$0.7 & 11K & 94.5$\pm$0.0 & 26.3$\pm$0.6 & 53.6$\pm$0.1 & 94K \\
Selective (on) & 92.3$\pm$0.8 & 2.6$\pm$0.9 & 24.7$\pm$0.4 & 12K & 94.9$\pm$0.2 & 26.7$\pm$1.1 & 53.0$\pm$0.6 & 95K \\
\midrule
Sep-Aware & 92.8$\pm$0.4 & 9.7$\pm$1.7 & 26.2$\pm$0.9 & 13K & 94.2$\pm$0.1 & 25.1$\pm$2.3 & 51.9$\pm$0.2 & 111K \\
\midrule
Two-Stage Sel. & 92.1$\pm$0.1 & 5.4$\pm$0.6 & \textbf{28.2$\pm$0.2} & 14K & 94.5$\pm$0.1 & \textbf{45.8$\pm$1.6} & \textbf{64.5$\pm$0.7} & 89K \\
\midrule
& \multicolumn{4}{c}{Fashion-MNIST} & \multicolumn{4}{c}{EMNIST-Letters} \\
\cmidrule(lr){2-5} \cmidrule(lr){6-9}
Method & \begin{tabular}[c]{@{}c@{}}Clean\\Acc.\end{tabular} & \begin{tabular}[c]{@{}c@{}}Acc$_{\mathrm{adv}}$\\(0.30)\end{tabular} & AURAC & \begin{tabular}[c]{@{}c@{}}\#\\Cat.\end{tabular} & \begin{tabular}[c]{@{}c@{}}Clean\\Acc.\end{tabular} & \begin{tabular}[c]{@{}c@{}}Acc$_{\mathrm{adv}}$\\(0.30)\end{tabular} & AURAC & \begin{tabular}[c]{@{}c@{}}\#\\Cat.\end{tabular} \\
\midrule
Vanilla & 80.7$\pm$0.5 & 0.9$\pm$0.2 & 24.4$\pm$0.6 & 57K & 80.8$\pm$0.6 & 12.2$\pm$0.1 & 37.1$\pm$1.0 & 114K \\
\midrule
AdvTrain (off) & 82.2$\pm$0.2 & 1.1$\pm$0.2 & 14.9$\pm$0.2 & 109K & 81.4$\pm$0.2 & 0.0$\pm$0.1 & 13.0$\pm$0.6 & 212K \\
AdvTrain (on) & 82.8$\pm$1.0 & 20.7$\pm$0.5 & 34.0$\pm$0.2 & 111K & 80.4$\pm$0.7 & 5.8$\pm$1.2 & 26.9$\pm$0.3 & 226K \\
\midrule
Selective (off) & 82.4$\pm$0.1 & 15.3$\pm$1.2 & 35.1$\pm$0.5 & 101K & 80.7$\pm$0.0 & 7.8$\pm$2.0 & 27.2$\pm$0.9 & 209K \\
Selective (on) & 81.1$\pm$0.4 & 17.0$\pm$0.8 & 34.8$\pm$0.1 & 101K & 80.7$\pm$0.7 & 6.4$\pm$0.7 & 25.8$\pm$0.6 & 210K \\
\midrule
Sep-Aware & 82.9$\pm$0.9 & 20.5$\pm$1.0 & 34.7$\pm$0.5 & 111K & 80.3$\pm$0.5 & 5.1$\pm$1.8 & 26.8$\pm$0.7 & 226K \\
\midrule
Two-Stage Sel. & 82.4$\pm$0.1 & \textbf{23.0$\pm$1.0} & \textbf{41.3$\pm$0.2} & 121K & 80.7$\pm$0.0 & \textbf{22.5$\pm$1.0} & \textbf{39.8$\pm$0.1} & 228K \\
\bottomrule
\end{tabular}
\vspace{0.3em}

\small{$\mathrm{Acc}_{\mathrm{adv}}(0.30)$ = adversarial accuracy (\%) at $\epsilon = 0.30$. AURAC (\%) = Area Under the Robust Accuracy Curve ($\epsilon \in [0.05, 0.35]$). \# Categories = category count (K = thousands). Bold = best per dataset.}
\end{table}

\begin{table}[t]
\centering
\caption{Defense Comparison under Black-Box Transfer PGD-20 ($\rho = 0.9$). Mean$\pm$std over 3 seeds}
\label{tab:bb-results}
\setlength{\tabcolsep}{2pt}
\footnotesize
\begin{tabular}{@{}l *{8}{c}@{}}
\toprule
& \multicolumn{4}{c}{USPS} & \multicolumn{4}{c}{MNIST} \\
\cmidrule(lr){2-5} \cmidrule(lr){6-9}
Method & \begin{tabular}[c]{@{}c@{}}Clean\\Acc.\end{tabular} & \begin{tabular}[c]{@{}c@{}}Acc$_{\mathrm{adv}}$\\(0.30)\end{tabular} & AURAC & \begin{tabular}[c]{@{}c@{}}\#\\Cat.\end{tabular} & \begin{tabular}[c]{@{}c@{}}Clean\\Acc.\end{tabular} & \begin{tabular}[c]{@{}c@{}}Acc$_{\mathrm{adv}}$\\(0.30)\end{tabular} & AURAC & \begin{tabular}[c]{@{}c@{}}\#\\Cat.\end{tabular} \\
\midrule
Vanilla & 92.2$\pm$0.1 & 9.1$\pm$0.3 & 48.4$\pm$0.2 & 6K & 94.5$\pm$0.1 & 85.8$\pm$0.4 & 92.1$\pm$0.2 & 56K \\
\midrule
AdvTrain (off) & 92.4$\pm$0.2 & \textbf{41.5$\pm$1.9} & \textbf{63.8$\pm$0.8} & 12K & 94.5$\pm$0.1 & 81.1$\pm$0.6 &
89.2$\pm$0.3 & 103K \\
AdvTrain (on) & 92.9$\pm$0.5 & 25.7$\pm$1.1 & 57.3$\pm$0.6 & 13K & 94.1$\pm$0.2 & 79.4$\pm$1.6 & 88.2$\pm$1.0 & 111K
\\
\midrule
Selective (off) & 92.1$\pm$0.1 & 29.2$\pm$0.4 & 58.3$\pm$0.2 & 11K & 94.5$\pm$0.0 & 88.7$\pm$1.2 & 93.3$\pm$0.2 & 94K
\\
Selective (on) & 92.5$\pm$0.2 & 24.0$\pm$0.5 & 56.3$\pm$0.3 & 12K & 94.5$\pm$0.1 & 89.2$\pm$0.3 & 93.5$\pm$0.2 & 95K
\\
\midrule
Sep-Aware & 92.8$\pm$0.5 & 25.8$\pm$0.3 & 57.8$\pm$0.3 & 13K & 94.2$\pm$0.2 & 78.0$\pm$1.9 & 87.5$\pm$0.7 & 111K \\
\midrule
Two-Stage Sel. & 92.1$\pm$0.1 & 30.0$\pm$0.7 & 59.4$\pm$0.2 & 14K & 94.5$\pm$0.0 & \textbf{89.4$\pm$0.4} &
\textbf{93.7$\pm$0.2} & 89K \\
\midrule
& \multicolumn{4}{c}{Fashion-MNIST} & \multicolumn{4}{c}{EMNIST-Letters} \\
\cmidrule(lr){2-5} \cmidrule(lr){6-9}
Method & \begin{tabular}[c]{@{}c@{}}Clean\\Acc.\end{tabular} & \begin{tabular}[c]{@{}c@{}}Acc$_{\mathrm{adv}}$\\(0.30)\end{tabular} & AURAC & \begin{tabular}[c]{@{}c@{}}\#\\Cat.\end{tabular} & \begin{tabular}[c]{@{}c@{}}Clean\\Acc.\end{tabular} & \begin{tabular}[c]{@{}c@{}}Acc$_{\mathrm{adv}}$\\(0.30)\end{tabular} & AURAC & \begin{tabular}[c]{@{}c@{}}\#\\Cat.\end{tabular} \\
\midrule
Vanilla & 80.8$\pm$0.2 & 67.8$\pm$0.5 & 79.8$\pm$0.3 & 57K & 80.7$\pm$0.0 & \textbf{89.1$\pm$1.0} &
\textbf{92.8$\pm$0.3} & 114K \\
\midrule
AdvTrain (off) & 82.2$\pm$0.3 & \textbf{72.5$\pm$0.2} & \textbf{82.1$\pm$0.1} & 109K & 81.4$\pm$0.2 & 70.1$\pm$2.0 &
80.5$\pm$0.6 & 212K \\
AdvTrain (on) & 82.8$\pm$1.2 & 68.6$\pm$1.6 & 79.9$\pm$0.9 & 111K & 80.4$\pm$0.7 & 70.0$\pm$1.5 & 79.5$\pm$0.5 & 226K
\\
\midrule
Selective (off) & 82.4$\pm$0.1 & 67.5$\pm$0.2 & 79.9$\pm$0.2 & 101K & 80.7$\pm$0.0 & 82.3$\pm$2.7 & 89.9$\pm$1.1 &
209K \\
Selective (on) & 81.7$\pm$0.2 & 68.5$\pm$0.4 & 80.4$\pm$0.2 & 101K & 80.3$\pm$0.1 & 83.7$\pm$1.4 & 90.4$\pm$0.4 & 210K
\\
\midrule
Sep-Aware & 82.9$\pm$1.2 & 68.3$\pm$2.2 & 79.7$\pm$1.1 & 111K & 80.3$\pm$0.5 & 72.8$\pm$1.2 & 80.8$\pm$1.0 & 226K \\
\midrule
Two-Stage Sel. & 82.4$\pm$0.2 & 69.0$\pm$0.4 & 80.5$\pm$0.2 & 121K & 80.7$\pm$0.0 & 85.5$\pm$1.1 & 90.7$\pm$0.8 & 228K
\\
\bottomrule
\end{tabular}
\vspace{0.3em}

\small{Transfer attacks use the SimpleCNN surrogate trained on each dataset. $\mathrm{Acc}_{\mathrm{adv}}(0.30)$ =
adversarial accuracy (\%) at $\epsilon = 0.30$. AURAC (\%) = Area Under the Robust Accuracy Curve ($\epsilon \in
[0.05, 0.35]$). \# Categories = category count (K = thousands). Bold = best per dataset.}
\end{table}

\FloatBarrier

\begin{figure}[h]
\centering
\includegraphics[width=0.99\columnwidth]{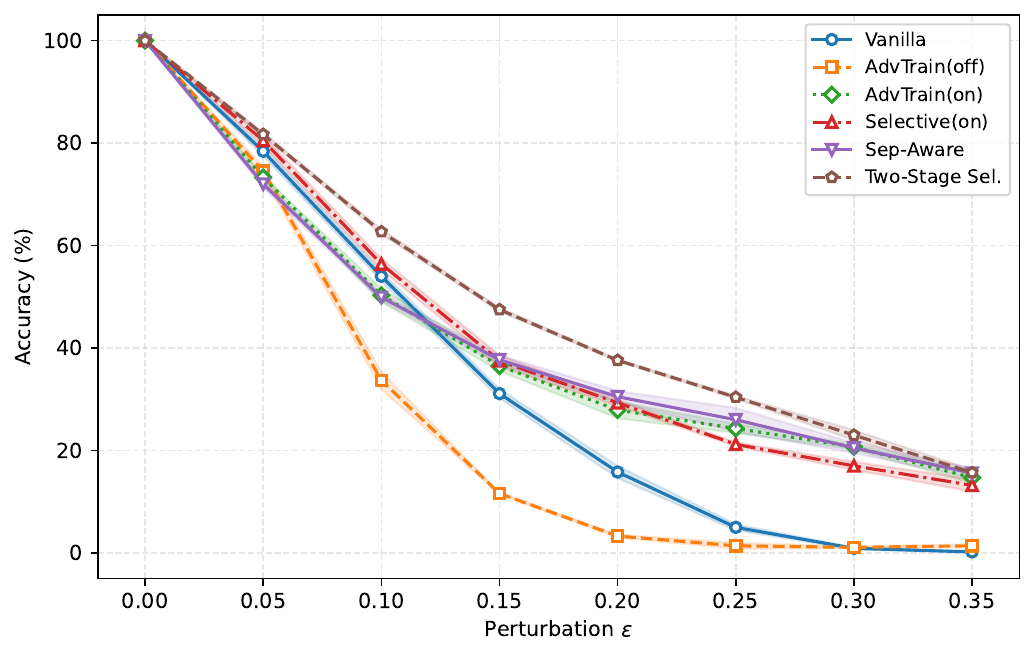}
\caption{Defense comparison on Fashion-MNIST under WB-Softmax PGD-20 ($\rho=0.9$). Two-stage selective training achieves the highest robust accuracy, while offline adversarial training collapses below vanilla. Mean over 3 seeds. For readability, the main text shows representative accuracy--$\epsilon$ curves for Fashion-MNIST; full tabular results for all datasets are reported in Tables~\ref{tab:main-results}--\ref{tab:bb-results}.}

\label{fig:defense-comparison}
\end{figure}

\begin{figure}[h]
\centering
\includegraphics[width=0.99\columnwidth]{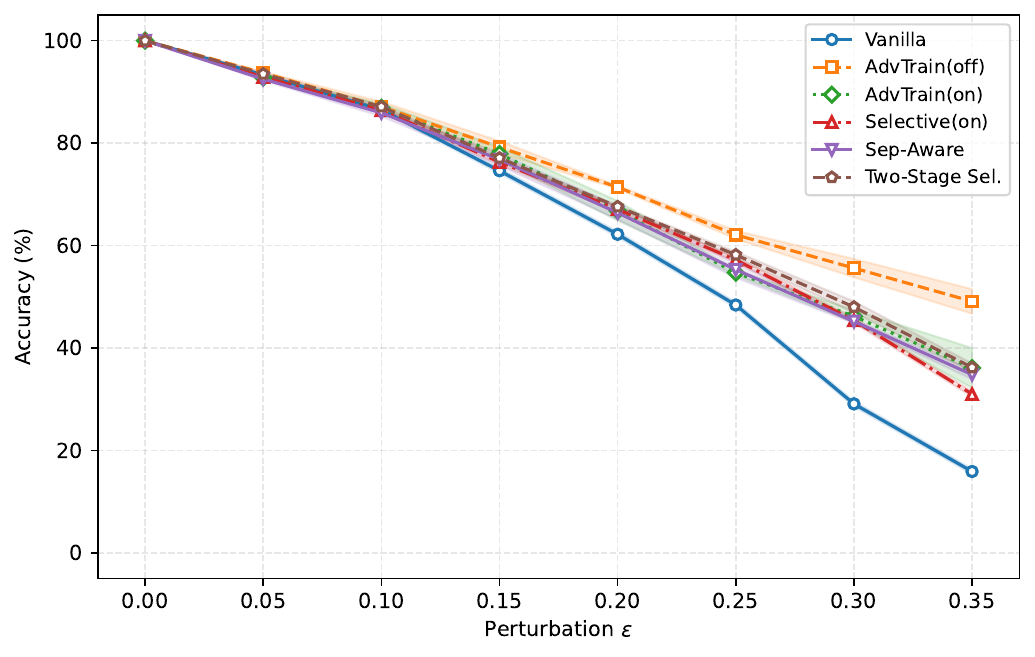}
\caption{Defense comparison on Fashion-MNIST under LRS Transfer PGD-20 ($\rho{=}0.9$). Offline adversarial training achieves highest accuracy under black-box transfer attacks, reversing its poor white-box performance. Mean over 3 seeds.}
\label{fig:defense-comparison-lrs}
\end{figure}

The clearest example is offline adversarial training. Under adaptive white-box evaluation, it collapses on all four datasets despite substantial category growth. Under transfer-based black-box evaluation, however, its behavior is dataset-dependent. On USPS and Fashion-MNIST, offline training attains the best black-box AURAC (63.8\% and 82.1\%), substantially exceeding vanilla (48.4\% and 79.8\%). By contrast, on MNIST and EMNIST-Letters, where vanilla already attains high black-box AURAC (92.1\% and 92.8\%), adversarial training degrades rather than improves transfer robustness. The USPS and Fashion-MNIST reversal directly illustrates Principle~1: in single-pass incremental learners, adversarial examples crafted once against a fixed clean-trained reference model can become stale after the classifier is further modified by adversarial training, because the final deployed decision boundary is no longer the one against which those perturbations were generated. As a result, offline adversarial training may learn perturbation patterns that transfer across models without conferring genuine robustness to adaptive attacks on the final adversarially trained model.

On EMNIST-Letters in particular, all six adversarial-training variants reduce black-box AURAC by 2–13 percentage points relative to vanilla, indicating that in regimes where the vanilla classifier already has high native black-box robustness, protocol choices that increase category count may widen rather than close the gap between category granularity and CNN-surrogate perturbation directions.

Online adversarial training removes the stale-attack mismatch of offline training, but its effect remains dataset-dependent. This indicates that regenerating perturbations on-the-fly is not sufficient; robustness also depends on which adversarial samples are learned and when they are introduced during streaming adaptation.

The most consistent replay-free gains come from two-stage selective training, which achieves the best overall white-box AURAC across all four datasets. This supports the broader conclusion that streaming robustness depends not only on attack strength but also on protocol design: progressive epsilon scheduling stabilizes adaptation, while selective filtering concentrates updates on genuinely vulnerable regions.

USPS provides an instructive case study where the diagnostics introduced in Section~\ref{sec:defense} are particularly visible. Geometry monitoring on USPS reveals a clear separation-collapse pattern: under selective adversarial training, minimum separation drops sharply while compactness remains high (Table~\ref{tab:separation-collapse}), indicating cross-class encroachment during adversarial updates to existing categories. This structural failure primarily harms high-$\epsilon$ robustness and motivates the separation-aware update rule in Algorithm~\ref{alg:sep-aware}.

\begin{table}[H]
\centering
\caption{Separation-collapse indicators on USPS ($\rho_a=0.9$). ``\# Categories'' denotes the number of learned categories. ``MS'' denotes the minimum-separation indicator in \eqref{eq:min-sep}. ``Compactness'' denotes the online iCVI compactness statistic computed from the current partition of learned categories/classes~\cite{moshtaghi2019online,dasilva2020icvi,dasilva2023icviartmap}. For this compactness statistic, smaller values indicate more concentrated within-partition structure, so the table should be interpreted jointly with MS rather than as an absolute stand-alone quality score.}
\label{tab:separation-collapse}
\setlength{\tabcolsep}{4pt}
\renewcommand{\arraystretch}{1.05}
\small
\begin{tabular}{@{}lccc@{}}
\toprule
Stage & \# Categories & MS & Compactness \\
\midrule
Before (clean only) & 6,058 & 0.036 & 0.994 \\
After (selective adv.) & 11,434 & 0.008 & 0.998 \\
\midrule
Change & $+$1.9$\times$ & $-$78\% & $+$0.4\% \\
\bottomrule
\end{tabular}
\end{table}

Separation-aware training, motivated by the separation-collapse diagnosis, matches AdvTrain (on) on USPS at $\epsilon=0.30$ (9.7$\pm$1.7\% vs.\ 10.6$\pm$1.2\%) and across category count and AURAC (13K cats, 26.2$\pm$0.9\% vs.\ 13K cats, 25.9$\pm$0.6\%); both exceed the selective family (5.4$\pm$0.6\% Two-Stage Sel., 2.6$\pm$0.9\% Selective on at $\epsilon=0.30$; Table~\ref{tab:usps-case-study}). The $\theta$ ablation (Table~\ref{tab:theta-ablation}) explains why: across $\theta\in[0.001,0.1]$, the overlap check rejects every absorption candidate, so the algorithm reduces to AdvTrain (on); larger $\theta$ admits cross-class hyperbox expansions and triggers separation collapse. Overlap-only gating therefore admits no operating point distinct from AdvTrain (on).

\begin{table}[H]
\centering
\caption{USPS High-$\epsilon$ Robustness (mean$\pm$std over 3 seeds).}
\label{tab:usps-case-study}
\small
\begin{tabular}{lccc}
\toprule
Method & \# Categories & Acc$_{\mathrm{adv}}$(0.25) & Acc$_{\mathrm{adv}}$(0.30) \\
\midrule
Selective (on) & 11.5K & 4.8$\pm$0.8 & 2.6$\pm$0.9 \\
Two-Stage Sel. & 13.7K & 11.7$\pm$1.0 & 5.4$\pm$0.6 \\
Sep-Aware & 13K & \textbf{11.7$\pm$1.9} & \textbf{9.7$\pm$1.7} \\
\bottomrule
\end{tabular}

\vspace{0.25em}
\footnotesize \emph{Acc$_{\mathrm{adv}}$(0.25) and Acc$_{\mathrm{adv}}$(0.30) denote adversarial accuracy (\%) at $\epsilon=0.25$ and $\epsilon=0.30$, respectively. Sep-Aware uses $\theta=0.01$. AdvTrain (on), shown in Table~\ref{tab:main-results}, achieves 10.6$\pm$1.2\% at $\epsilon=0.30$, statistically tied with Sep-Aware.}
\end{table}

\begin{table}[H]
\centering
\caption{Separation-Aware Threshold $\theta$ Ablation (USPS, $\rho{=}0.9$). Mean$\pm$std over 3 seeds.}
\label{tab:theta-ablation}
\small
\begin{tabular}{@{}lcccc@{}}
\toprule
$\theta$ & Clean & Adv@0.25 & Adv@0.30 & AURAC \\
\midrule
0.001 & 92.8$\pm$0.4 & 11.7$\pm$1.9 & 9.7$\pm$1.7 & 26.2$\pm$0.9 \\
0.005 & 92.8$\pm$0.4 & 11.7$\pm$1.9 & 9.7$\pm$1.7 & 26.2$\pm$0.9 \\
\textbf{0.01} & \textbf{92.8$\pm$0.4} & \textbf{11.7$\pm$1.9} & \textbf{9.7$\pm$1.7} & \textbf{26.2$\pm$0.9} \\
0.02 & 92.8$\pm$0.4 & 11.7$\pm$1.9 & 9.7$\pm$1.7 & 26.2$\pm$0.9 \\
0.05 & 92.8$\pm$0.4 & 11.7$\pm$1.9 & 9.7$\pm$1.7 & 26.2$\pm$0.9 \\
0.1   & 92.8$\pm$0.4 & 11.7$\pm$1.9 & 9.7$\pm$1.7 & 26.2$\pm$0.9 \\
\midrule
0.5   & 91.1$\pm$0.6 & 0.3$\pm$0.2 & 0.1$\pm$0.1 & 11.9$\pm$0.7 \\
1.0 (off) & 91.0$\pm$0.9 & 2.0$\pm$0.6 & 1.8$\pm$1.0 & 15.9$\pm$0.4 \\
\bottomrule
\end{tabular}

\vspace{0.3em}
\footnotesize
\emph{Adv@0.25} and \emph{Adv@0.30} denote adversarial accuracy (\%) at $\epsilon = 0.25$ and $\epsilon = 0.30$, respectively. AURAC (\%) = Area Under the Robust Accuracy Curve. Bold = default setting. Values for $\theta \in [0.001, 0.1]$ are bit-identical: the overlap check rejects every absorption candidate, so the update reduces to new-category creation (the AdvTrain-on path). At $\theta=0.5$ the check admits roughly 23\% of absorptions and induces separation collapse; at $\theta=1.0$ it is effectively disabled.
\end{table}

Table~\ref{tab:theta-ablation} characterizes the constraint across $\theta$. Across $\theta \in [0.001, 0.1]$, the overlap check rejects every absorption candidate, so the update reduces to the new-category creation path used by standard online adversarial training. Loosening to $\theta=0.5$ admits roughly 23\% of absorptions and triggers separation collapse (AURAC drops to 11.9\%); $\theta=1.0$ effectively disables the check (${\sim}83\%$ absorption) and yields 15.9\% AURAC. Across $\theta$, no operating point yields robustness beyond AdvTrain (on); overlap-only absorption gating in this form does not differentiate from AdvTrain (on) on streaming Fuzzy ARTMAP.

The second main empirical phenomenon is match-score ordering reversal. On vanilla models, adversarial examples tend to produce lower match scores than clean samples on USPS (0.987 vs.\ 0.997), which suggests that match could be used as a rejection signal. After selective adversarial training, however, that ordering can invert (0.983 vs.\ 0.968), and the corresponding match-threshold detector degrades sharply: AUC falls from 0.72 to 0.38, below chance. This observation is consistent with the theoretical development in Section~\ref{sec:defense}. Lemma~\ref{lem:match-preservation} establishes that selective absorption preserves the absorbed-category match locally, while Proposition~\ref{prop:match-nonmonotone} explains why post-training match need not remain monotonic with correctness globally. Empirically, the USPS inversion shows that this is not merely a theoretical possibility: internal scores that are reliable in the vanilla regime can become unreliable after adversarial training. In deployment terms, rejection, abstention, or escalation rules that treat match as a trust signal must therefore be recalibrated jointly with the training protocol rather than transferred unchanged from the vanilla model.

Because the main robustness results in Tables~\ref{tab:main-results}--\ref{tab:bb-results} are reported on the clean-correct subset, Appendix~B provides a derived unconditional view by converting the reported clean accuracy and conditional robustness into unconditional robust accuracy and unconditional AURAC. The results show that the main qualitative conclusions remain unchanged: two-stage selective training remains the strongest overall replay-free defense, and separation-aware training remains statistically tied with standard online adversarial training at high~\(\epsilon\) on USPS. Thus, the main ranking claims are not merely an artifact of conditioning on correct clean prediction.

\subsection{Additional Analyses and Verification}

Table~\ref{tab:attack-strength} further verifies that the proposed WB-Softmax threat model is a strong adaptive white-box evaluation protocol. On vanilla models at $\rho_a=0.9$, WB-Softmax consistently exceeds or matches the strongest black-box baselines overall. Transfer remains relatively strong on USPS, but is much weaker on MNIST and EMNIST-Letters and only partially competitive on Fashion-MNIST. These results support the interpretation that the white-box findings above reflect genuine vulnerability rather than weak optimization.

\noindent\textbf{WB-Softmax ablations.}
Across datasets, increasing PGD steps monotonically strengthens the attack and largely saturates by 20 steps, which motivates PGD-20 as the default evaluation setting. The temperature $\tau$ controls the tradeoff between gradient smoothness and winner concentration. Table~\ref{tab:tau-ablation} reports attack success as a function of $\tau$ on USPS. Smaller $\tau$ values sharpen the softmax distribution and strengthen the attack; at the same time, very small temperatures can create numerical instability. At $\tau=0.01$, the attack achieves 99.4\% success at $\epsilon=0.30$, essentially matching the strongest settings while remaining numerically stable. This supports our choice of $\tau=0.01$ as a robust default.

\begin{table}[H]
\centering
\caption{WB-Softmax Attack Strength vs.\ Temperature $\tau$ (USPS, PGD-20)}
\label{tab:tau-ablation}
\small
\begin{tabular}{@{}lccc@{}}
\toprule
$\tau$ & $\epsilon{=}0.20$ & $\epsilon{=}0.30$ & $\epsilon{=}0.35$ \\
\midrule
0.005 & 94.7 & 98.9 & 99.4 \\
\textbf{0.01} & \textbf{93.8} & \textbf{99.4} & \textbf{99.8} \\
0.02 & 77.3 & 97.0 & 98.1 \\
0.05 & 60.1 & 90.0 & 94.3 \\
0.10 & 54.1 & 86.0 & 93.0 \\
\bottomrule
\end{tabular}
\vspace{0.3em}

\small{Attack success rate (\%) on $n{=}500$ clean-correct samples. Lower $\tau$ concentrates probability on the winner, strengthening the attack. Bold = default setting.}
\end{table}

\begin{table}[H]
\centering
\caption{Attack Strength Comparison (Vanilla Models, $\rho=0.9$, $\epsilon=0.35$)}
\label{tab:attack-strength}
\small
\begin{tabular}{lcccc}
\toprule
Attack & USPS & MNIST & FMNIST & EMNIST-L \\
\midrule
WB-Softmax PGD & 100 & 89 & 100 & 96 \\
BB Transfer PGD & 97 & 18 & 41 & 12 \\
LRS Transfer PGD & 98 & 19 & 84 & 26 \\
Square (query) & 64 & 28 & 61  & 80 \\
\bottomrule
\end{tabular}
\vspace{0.3em}

\small{Attack success rate (\%) on $n{=}1000$ clean-correct samples. PGD uses 20 steps with step size $\eta = \epsilon/4$. Square Attack uses 5000 queries/sample.}
\end{table}

Category growth alone does not explain robustness. All adversarial training variants increase category count, but the resulting robustness varies sharply. The clearest counterexample is offline adversarial training on MNIST: it creates 103K categories, among the largest models considered, yet yields the worst AURAC (14.8\%). By contrast, two-stage selective training achieves the best robustness with more moderate growth. This indicates that, in ARTMAP, robustness depends more on how categories are created, updated, and scheduled than on how many categories are ultimately formed.

\noindent\textbf{Robustness-per-compute tradeoff.}
Table~\ref{tab:robustness-compute} compares training cost (wall-clock time and category count) against robustness (AURAC under FGSM and PGD-20 evaluation) for selective training with varying attack strengths during training. Stronger training attacks increase both computational cost and resulting robustness, but with diminishing returns. PGD-20 selective training attains the highest AURAC against PGD-20 (22.1\%) at 78.9\,s, whereas PGD-10 attains 19.4\% at 49.7\,s and FGSM attains 19.0\% at only 17.3\,s. These results suggest that moderate training attack strengths (PGD-5 to PGD-10) may offer a favorable efficiency--robustness tradeoff in deployment scenarios where training cost is constrained.

\begin{table}[H]
\centering
\caption{Robustness-per-Compute Tradeoff (USPS, Selective Training)}
\label{tab:robustness-compute}
\small
\begin{tabular}{@{}lcccc@{}}
\toprule
Training & Time (s) & \#Cat. & \multicolumn{2}{c}{AURAC (\%)} \\
\cmidrule(lr){4-5}
Attack &  &  & vs FGSM & vs PGD-20 \\
\midrule
Vanilla & 7.4  & 6.2K  & 72.0 & 18.6 \\
FGSM    & 17.3 & 7.2K  & 72.2 & 19.0 \\
PGD-5   & 34.1 & 10.9K & 73.1 & 17.5 \\
PGD-10  & 49.7 & 11.4K & 73.5 & 19.4 \\
PGD-20  & 78.9 & 11.5K & 73.5 & \textbf{22.1} \\
\bottomrule
\end{tabular}
\vspace{0.3em}

\small{Mean over 3 seeds. Time = training wall-clock time. \#Cat. = category count. AURAC is evaluated against FGSM and PGD-20 attacks over $\epsilon \in [0.05, 0.35]$.}
\end{table}

Finally, Table~\ref{tab:sanity-checks} verifies that the reported robustness is not an artifact of gradient masking. All four adapted checks are satisfied on all four datasets: iterative PGD is at least as strong as FGSM, adaptive white-box attacks match or exceed transfer attacks, the smooth WB-Softmax surrogate is stronger than hard winner-take-all optimization, and accuracy degrades smoothly as $\epsilon$ increases. Together with the attack-strength results above, these checks support the validity of the empirical conclusions in this section.

\begin{table}[H]
\centering
\caption{Anti-Gradient-Masking Sanity Checks ($\rho{=}0.9$)}
\label{tab:sanity-checks}
\small
\begin{tabular}{@{}lcccc@{}}
\toprule
Sanity Check & USPS & MNIST & FMNIST & EMNIST-L\\
\midrule
PGD $\geq$ FGSM & \checkmark & \checkmark & \checkmark & \checkmark \\
WB $\geq$ BB transfer & \checkmark & \checkmark & \checkmark & \checkmark \\
Smooth $>$ Hard ops & \checkmark & \checkmark & \checkmark & \checkmark \\
Smooth acc-vs-$\epsilon$ & \checkmark & \checkmark & \checkmark & \checkmark \\
\bottomrule
\end{tabular}
\vspace{0.3em}

\small{Evaluated on vanilla models at $\epsilon{=}0.35$. WB = WB-Softmax, BB = black-box.}
\end{table}

\FloatBarrier

\section{Conclusion}
\label{sec:conclusion}

We presented a systematic study of adversarial robustness in Fuzzy ARTMAP under strict single-pass streaming constraints. The central message is that robustness in this regime cannot be studied by simply transplanting the methodology of offline deep networks. Because ARTMAP predicts through winner-take-all category competition and evolves continuously through replay-free updates, both attack construction and robustness evaluation must be aligned with the model’s native mechanism and final streamed state.

Our results support this claim along three dimensions. First, WB-Softmax provides a strong mechanism-aligned adaptive white-box evaluator, showing that meaningful gradients and severe vulnerabilities remain present despite ARTMAP’s non-smooth competition structure. Second, robustness outcomes depend critically on protocol design in streaming learners: offline adversarial training can appear effective under transfer-based evaluation while collapsing under adaptive white-box attacks on the final deployed model, whereas progressive two-stage selective training yields the strongest overall replay-free robustness across USPS, MNIST, Fashion-MNIST, and EMNIST-Letters. Third, ART’s explicit category geometry is not only interpretable but operationally useful for diagnostics: geometry monitoring reveals separation collapse as a structural failure mode, and match-score analysis exposes a distinct semantic reliability failure (match-score inversion). The geometric diagnosis admits a natural absorption-gating rule (Sep-Aware), and we characterize its behavior fully: at the recommended $\theta$, the rule's update path coincides with AdvTrain (on) because adversarial examples lie in regions where absorption-driven hyperbox expansion crosses class boundaries. The diagnostic framework therefore makes the structural constraint on overlap-only gating explicit and motivates the alternative gating criteria discussed below.

Taken together, these findings show that adversarial robustness in streaming prototype-based learners is simultaneously a problem of attack alignment, protocol design, and post-training reliability. More broadly, the framework developed here suggests that interpretable internal structure can play a dual role in streaming robustness research: it can improve diagnosis of failure modes and also support targeted replay-free interventions that would be difficult to design in black-box models.

Several directions remain open. First, our study focuses on $\ell_\infty$ attacks and controlled small-to-medium image benchmarks; extending the analysis to larger datasets, higher-dimensional inputs, real-world streams, and additional modalities is an important next step. Another important direction is cross-architecture transfer from ART-generated adversarial examples to offline deep models. The present work focuses on the inverse direction---using deep surrogate transfer attacks as black-box evaluators for ARTMAP---and on adaptive white-box evaluation of the final streamed ARTMAP model. Studying whether adversarial examples generated from ART category geometry transfer to offline convolutional architectures would further clarify the interaction between prototype-based streaming learners and mainstream deep models. Second, certified robustness, distributionally robust guarantees, and tighter theoretical characterizations of geometry evolution in ART remain largely unexplored. Third, the structural limitation of overlap-only gating motivates richer absorption criteria: margin-based gating, delta-overlap constraints (rejecting only absorptions that \emph{increase} cross-class overlap), and non-geometric criteria such as match-score statistics are natural next interventions within the diagnosis-to-rule framework. Finally, the observed match-score inversion suggests a broader research direction at the intersection of streaming robustness and deployment reliability: understanding when interpretable internal scores remain valid after adversarial adaptation, and how they should be recalibrated when they do not.

\section*{Acknowledgment}

This research was sponsored by the Army Research Laboratory and was accomplished under Cooperative Agreement Number W911NF-22-2-0209.
This research was supported by NSF grant 2420248 and by the Kummer Institute, Mary Finley Endowment, and Intelligent Systems Center of the Missouri University of Science and Technology.

The views and conclusions contained in this document are those of the authors and should not be interpreted as representing the official policies, either expressed or implied, of the Army Research Laboratory or the U.S. Government.
The U.S. Government is authorized to reproduce and distribute reprints for Government purposes, notwithstanding any copyright notation herein.

The computation for this work was performed on the high-performance computing infrastructure provided by Research Support Solutions at Missouri University of Science and Technology https://doi.org/10.71674/PH64-N397

\bibliographystyle{elsarticle-num}
\bibliography{references}

\newpage

\appendix
\section*{Appendix}
\addcontentsline{toc}{section}{Appendix: Supplementary Material}
This appendix provides additional theoretical and empirical details referenced in the main paper. Specifically, it includes: (i) a constructive proof of Proposition~1, and (ii) auxiliary derived unconditional robustness tables corresponding to the conditional clean-correct evaluation reported in the main text.

\section{Proof of Proposition 1}
\label{app:proof-match}

\textit{Proof.} We provide a constructive argument for why selective adversarial training can break the monotone relationship between post-training match score and correctness.

\textit{Setup.}
Let
\begin{equation}
I(\bm{x}) = [\bm{x};\, 1-\bm{x}]
\label{eq:supp-complement}
\end{equation}
denote the complement-coded input, and let category $j$ have weight vector $\bm{w}_j$. The ARTMAP category-level match function is
\begin{equation}
M_j(I) = \frac{|I \wedge \bm{w}_j|}{|I|}.
\label{eq:supp-match}
\end{equation}
Under complement coding, $|I(\bm{x})|$ is constant for a fixed input dimension, so comparisons of category-level match reduce to comparisons of the numerator $|I \wedge \bm{w}_j|$.

Under fast learning, when a sample with code $I$ is absorbed by category $j$, the update is
\begin{equation}
\bm{w}_j^{\mathrm{new}} = \bm{w}_j \wedge I,
\label{eq:supp-fast-learning}
\end{equation}
which can be interpreted as shrinking the category hyperbox toward the presented sample in complement-coded space.

\textit{Step 1: Selective training preferentially updates on misclassified adversarial samples.}
By definition of selective adversarial training, only adversarial samples $\bm{x}'$ with
\begin{equation}
f(\bm{x}') \neq y
\label{eq:supp-misclassified}
\end{equation}
trigger an update using label $y$, either by absorption into an existing correct-class category or by creation of a new correct-class category. Clean samples that are already correctly classified do not trigger comparable updates under the same rule.

\textit{Step 2: In the absorbed-category case, match is preserved.}
Step 2 establishes the match-preservation lemma for the absorbed-category case. Consider a misclassified adversarial sample $\bm{x}'$ that is absorbed into an existing correct-class category $j$. For that same adversarial code $I(\bm{x}')$, after the update we have
\begin{equation}
I(\bm{x}') \wedge \bm{w}_j^{\mathrm{new}}
=
I(\bm{x}') \wedge (\bm{w}_j \wedge I(\bm{x}'))
=
I(\bm{x}') \wedge \bm{w}_j.
\label{eq:supp-absorb-identity}
\end{equation}
Therefore,
\begin{equation}
\begin{aligned}
M_j^{\mathrm{new}}(I(\bm{x}'))
&=
\frac{|I(\bm{x}') \wedge \bm{w}_j^{\mathrm{new}}|}{|I(\bm{x}')|} \\
&=
\frac{|I(\bm{x}') \wedge \bm{w}_j|}{|I(\bm{x}')|} \\
&=
M_j^{\mathrm{old}}(I(\bm{x}')).
\end{aligned}
\label{eq:supp-match-preservation}
\end{equation}
Hence, for the absorbed-category case, fast learning preserves the category-level match of the updated adversarial sample exactly:
\begin{equation}
M_j^{\mathrm{new}}(I(\bm{x}')) = M_j^{\mathrm{old}}(I(\bm{x}')).
\label{eq:supp-preserved}
\end{equation}

In the special case of category creation, if a new correct-class category is initialized at
\begin{equation}
\bm{w}_{j_{\mathrm{new}}}^{\mathrm{new}} = I(\bm{x}'),
\label{eq:supp-category-creation}
\end{equation}
then
\begin{equation}
M_{j_{\mathrm{new}}}(I(\bm{x}')) = 1.
\label{eq:supp-newcat-match}
\end{equation}
Thus, while absorbed-category match is preserved, higher attained match values can arise through category creation rather than ordinary absorption.

\textit{Step 3: Preserved or high match does not imply correctness.}
Match is a geometric quantity defined relative to category structure, whereas the final prediction $f(\cdot)$ is determined by winner-take-all competition together with the map field. Consequently, correctness is not determined by category-level match alone.

In particular, after selective adversarial training, there can exist regimes in which:
\begin{enumerate}
    \item an input $\bm{x}_1$ attains relatively high post-training match to an absorbed or newly created category, yet is still misclassified because a competing category wins the global competition or maps to a different class; while
    \item another input $\bm{x}_2$ attains lower post-training match, yet is correctly classified because the winning category and map-field assignment are favorable.
\end{enumerate}
Therefore, it is possible to have

\[
M_{j_1}(I(\bm{x}_1)) > M_{j_2}(I(\bm{x}_2)), \qquad
f(\bm{x}_1)\neq y_1, \qquad
f(\bm{x}_2)=y_2.
\]

Combining Steps 1--3 yields the claim of Proposition~1: after selective adversarial training, post-training match need not remain a monotone proxy for correctness. Consequently, match-threshold rejection rules calibrated on vanilla models need not remain valid after selective adversarial training.
\hfill$\square$

\section{Derived Unconditional Robustness Tables}
\label{app:unconditional}

The main robustness results in Tables~\ref{tab:main-results}--\ref{tab:bb-results} of the main paper are reported on the clean-correct subset. To provide an additional reference view, we derive unconditional robust accuracy and unconditional AURAC directly from the reported clean accuracy and conditional robustness values.

Let
\begin{equation}
\mathrm{CondRob}(\epsilon)=P\!\left(f(\bm{x}^{\mathrm{adv}}_{\epsilon})=y \mid f(\bm{x})=y\right)
\label{eq:supp-condrob}
\end{equation}
denote the conditional robust accuracy at perturbation level $\epsilon$, and let
\begin{equation}
\mathrm{CleanAcc}=P(f(\bm{x})=y)
\label{eq:supp-cleanacc}
\end{equation}
denote the clean accuracy. Then the corresponding unconditional robust accuracy is
\begin{equation}
\begin{aligned}
\mathrm{UncondRob}(\epsilon)
&=
P\!\big(f(\bm{x})=y,\ f(\bm{x}^{\mathrm{adv}}_{\epsilon})=y\big) \\
&=
\mathrm{CleanAcc}\times \mathrm{CondRob}(\epsilon).
\end{aligned}
\label{eq:supp-uncondrob}
\end{equation}

Accordingly, for the reported point metric at $\epsilon = 0.30$,
\begin{equation}
\mathrm{Uncond.Adv@0.30}=\frac{\mathrm{CleanAcc}\times \mathrm{Adv@0.30}}{100},
\label{eq:supp-uncond-adv030}
\end{equation}
where all quantities are expressed in percent.

Similarly, since the reported AURAC is computed from the robust-accuracy curve over $\epsilon \in [0.05, 0.35]$, the corresponding unconditional AURAC is derived as
\begin{equation}
\mathrm{Uncond.AURAC}=\frac{\mathrm{CleanAcc}\times \mathrm{AURAC}}{100}.
\label{eq:supp-uncond-aurac}
\end{equation}

The tables below are therefore auxiliary derived indicators computed from the reported means in Tables~\ref{tab:main-results}--\ref{tab:bb-results} of the main paper. They should be interpreted as a complementary unconditional view rather than re-evaluated experimental measurements.

\begin{table}[t]
\centering
\caption{Derived Unconditional White-Box Robustness}
\label{tab:appendix-uncond-wb}
\setlength{\tabcolsep}{4pt}
\small
\begin{tabular}{l cc cc}
\toprule
& \multicolumn{2}{c}{USPS} & \multicolumn{2}{c}{MNIST} \\
\cmidrule(lr){2-3} \cmidrule(lr){4-5}
Method & Uncond.Adv@0.30 & Uncond.AURAC & Uncond.Adv@0.30 & Uncond.AURAC \\
\midrule
Vanilla         & 0.46  & 20.03 & 30.08 & 58.36 \\
AdvTrain (off)  & 0.28  & 11.64 & 0.00  & 13.99 \\
AdvTrain (on)   & \textbf{9.85}  & 24.06 & 19.86 & 48.18 \\
Selective (off) & 2.77  & 23.42 & 24.85 & 50.65 \\
Selective (on)  & 2.40  & 22.80 & 25.34 & 50.30 \\
Sep-Aware       & 9.00  & 24.31 & 23.64 & 48.89 \\
Two-Stage Sel.  & 4.97  & \textbf{25.97} & \textbf{43.28} & \textbf{60.95} \\
\midrule
& \multicolumn{2}{c}{Fashion-MNIST} & \multicolumn{2}{c}{EMNIST-Letters} \\
\cmidrule(lr){2-3} \cmidrule(lr){4-5}
Method & Uncond.Adv@0.30 & Uncond.AURAC & Uncond.Adv@0.30 & Uncond.AURAC \\
\midrule
Vanilla         & 0.73  & 19.69 & 9.86  & 29.97 \\
AdvTrain (off)  & 0.90  & 12.25 & 0.00  & 10.58 \\
AdvTrain (on)   & 17.14 & 28.15 & 4.66  & 21.63 \\
Selective (off) & 12.61 & 28.92 & 6.29  & 21.95 \\
Selective (on)  & 13.79 & 28.22 & 5.17  & 20.82 \\
Sep-Aware       & 16.99 & 28.77 & 4.10  & 21.52 \\
Two-Stage Sel.  & \textbf{18.95} & \textbf{34.03} & \textbf{18.16} & \textbf{32.12} \\
\bottomrule
\end{tabular}
\vspace{0.3em}

\small{Derived from Table~\ref{tab:main-results} of the main paper using
$\mathrm{Uncond.Adv@0.30}=\mathrm{Clean}\times \mathrm{Adv@0.30}/100$
and
$\mathrm{Uncond.AURAC}=\mathrm{Clean}\times \mathrm{AURAC}/100$.
Values are computed from reported means and are therefore auxiliary derived indicators rather than re-evaluated experimental measurements. Bold = best per dataset.}
\end{table}

\begin{table}[t]
\centering
\caption{Derived Unconditional Black-Box Transfer Robustness}
\label{tab:appendix-uncond-bb}
\setlength{\tabcolsep}{4pt}
\small
\begin{tabular}{l cc cc}
\toprule
& \multicolumn{2}{c}{USPS} & \multicolumn{2}{c}{MNIST} \\
\cmidrule(lr){2-3} \cmidrule(lr){4-5}
Method & Uncond.Adv@0.30 & Uncond.AURAC & Uncond.Adv@0.30 & Uncond.AURAC \\
\midrule
Vanilla         & 8.39  & 44.62 & 81.08 & 87.03 \\
AdvTrain (off)  & \textbf{38.35} & \textbf{58.95} & 76.64 & 84.29 \\
AdvTrain (on)   & 23.88 & 53.23 & 74.72 & 83.00 \\
Selective (off) & 26.89 & 53.69 & 83.82 & 88.17 \\
Selective (on)  & 22.20 & 52.08 & 84.29 & 88.36 \\
Sep-Aware       & 23.94 & 53.64 & 73.48 & 82.42 \\
Two-Stage Sel.  & 27.63 & 54.71 & \textbf{84.48} & \textbf{88.55} \\
\midrule
& \multicolumn{2}{c}{Fashion-MNIST} & \multicolumn{2}{c}{EMNIST-Letters} \\
\cmidrule(lr){2-3} \cmidrule(lr){4-5}
Method & Uncond.Adv@0.30 & Uncond.AURAC & Uncond.Adv@0.30 & Uncond.AURAC \\
\midrule
Vanilla         & 55.73 & 65.60 & \textbf{71.90} & \textbf{74.89} \\
AdvTrain (off)  & \textbf{59.59} & \textbf{67.49} & 57.06 & 65.53 \\
AdvTrain (on)   & 56.80 & 66.16 & 56.28 & 63.92 \\
Selective (off) & 55.62 & 65.84 & 66.42 & 72.55 \\
Selective (on)  & 55.96 & 65.69 & 67.21 & 72.59 \\
Sep-Aware       & 56.62 & 66.07 & 58.46 & 64.88 \\
Two-Stage Sel.  & 56.86 & 66.33 & 69.00 & 73.20 \\
\bottomrule
\end{tabular}
\vspace{0.3em}

\small{Derived from the clean accuracy and conditional robustness values reported in Table~\ref{tab:bb-results} of the main paper using
$\mathrm{Uncond.Adv@0.30}=\mathrm{Clean}\times \mathrm{Adv@0.30}/100$
and
$\mathrm{Uncond.AURAC}=\mathrm{Clean}\times \mathrm{AURAC}/100$.
Values are computed from reported means and are therefore auxiliary derived indicators rather than re-evaluated experimental measurements. Bold = best per dataset.}
\end{table}

Tables~\ref{tab:appendix-uncond-wb}--\ref{tab:appendix-uncond-bb} indicate that the main qualitative conclusions of the paper are preserved under this derived unconditional view. Under white-box evaluation, two-stage selective training is the strongest overall replay-free defense, achieving the best mean unconditional AURAC on all four datasets and the best mean unconditional robustness at $\epsilon=0.30$ on MNIST, Fashion-MNIST, and EMNIST-Letters; on USPS, online adversarial training (AdvTrain (on)) achieves the best mean unconditional robustness at $\epsilon=0.30$, statistically tied with separation-aware training. Under black-box transfer evaluation, the same protocol-dependent ranking behavior remains visible: offline adversarial training is strongest on USPS and Fashion-MNIST, two-stage selective training is strongest on MNIST, and vanilla is strongest on EMNIST-Letters. These auxiliary tables therefore support the claim that the main ranking conclusions are not merely artifacts of conditioning on clean-correct samples.

\end{document}